\documentclass{article}

\makeatletter
\relax

\usepackage[verbose=true,letterpaper]{geometry}
\AtBeginDocument{
  \newgeometry{
    textheight=9in,
    textwidth=6.5in,
    top=1in,
    headheight=14pt,
    headsep=25pt,
    footskip=30pt
  }
}

\newcommand{\headeright}{A Preprint}
\newcommand{\undertitle}{A Preprint}
\newcommand{\shorttitle}{\@title}

\usepackage{fancyhdr}
\fancyheadoffset{0pt}
\def\keywordname{{\bfseries \emph{Keywords}}}%
\def\keywords#1{\par\addvspace\medskipamount{\rightskip=0pt plus1cm
\def\and{\ifhmode\unskip\nobreak\fi\ $\cdot$
}\noindent\keywordname\enspace\ignorespaces#1\par}}

\renewcommand{\normalsize}{%
  \@setfontsize\normalsize\@xpt\@xipt
  \abovedisplayskip      7\p@ \@plus 2\p@ \@minus 5\p@
  \abovedisplayshortskip \z@ \@plus 3\p@
  \belowdisplayskip      \abovedisplayskip
  \belowdisplayshortskip 4\p@ \@plus 3\p@ \@minus 3\p@
}
\normalsize
\renewcommand{\small}{%
  \@setfontsize\small\@ixpt\@xpt
  \abovedisplayskip      6\p@ \@plus 1.5\p@ \@minus 4\p@
  \abovedisplayshortskip \z@  \@plus 2\p@
  \belowdisplayskip      \abovedisplayskip
  \belowdisplayshortskip 3\p@ \@plus 2\p@   \@minus 2\p@
}
\renewcommand{\footnotesize}{\@setfontsize\footnotesize\@ixpt\@xpt}
\renewcommand{\scriptsize}{\@setfontsize\scriptsize\@viipt\@viiipt}
\renewcommand{\tiny}{\@setfontsize\tiny\@vipt\@viipt}
\renewcommand{\large}{\@setfontsize\large\@xiipt{14}}
\renewcommand{\Large}{\@setfontsize\Large\@xivpt{16}}
\renewcommand{\LARGE}{\@setfontsize\LARGE\@xviipt{20}}
\renewcommand{\huge}{\@setfontsize\huge\@xxpt{23}}
\renewcommand{\Huge}{\@setfontsize\Huge\@xxvpt{28}}

\providecommand{\section}{}
\renewcommand{\section}{%
  \@startsection{section}{1}{\z@}%
                {-2.0ex \@plus -0.5ex \@minus -0.2ex}%
                { 1.5ex \@plus  0.3ex \@minus  0.2ex}%
                {\large\bf\raggedright}%
}
\providecommand{\subsection}{}
\renewcommand{\subsection}{%
  \@startsection{subsection}{2}{\z@}%
                {-1.8ex \@plus -0.5ex \@minus -0.2ex}%
                { 0.8ex \@plus  0.2ex}%
                {\normalsize\bf\raggedright}%
}
\providecommand{\subsubsection}{}
\renewcommand{\subsubsection}{%
  \@startsection{subsubsection}{3}{\z@}%
                {-1.5ex \@plus -0.5ex \@minus -0.2ex}%
                { 0.5ex \@plus  0.2ex}%
                {\normalsize\bf\raggedright}%
}
\providecommand{\paragraph}{}
\renewcommand{\paragraph}{%
  \@startsection{paragraph}{4}{\z@}%
                {1.5ex \@plus 0.5ex \@minus 0.2ex}%
                {-1em}%
                {\normalsize\bf}%
}
\providecommand{\subparagraph}{}
\renewcommand{\subparagraph}{%
  \@startsection{subparagraph}{5}{\z@}%
                {1.5ex \@plus 0.5ex \@minus 0.2ex}%
                {-1em}%
                {\normalsize\bf}%
}

\renewcommand{\topfraction      }{0.85}
\renewcommand{\bottomfraction   }{0.4}
\renewcommand{\textfraction     }{0.1}
\renewcommand{\floatpagefraction}{0.7}

\newlength{\@abovecaptionskip}
\newlength{\@belowcaptionskip}

\renewenvironment{table}
  {\setlength{\abovecaptionskip}{\@belowcaptionskip}%
   \setlength{\belowcaptionskip}{\@abovecaptionskip}%
   \@float{table}}
  {\end@float}

\renewcommand{\footnoterule}{\kern-3\p@ \hrule width 12pc \kern 2.6\p@}
\def\@listi  {\leftmargin\leftmargini}
\def\@listii {\leftmargin\leftmarginii
              \labelwidth\leftmarginii
              \advance\labelwidth-\labelsep
              \topsep  2\p@ \@plus 1\p@    \@minus 0.5\p@
              \parsep  1\p@ \@plus 0.5\p@ \@minus 0.5\p@
              \itemsep \parsep}
\def\@listiii{\leftmargin\leftmarginiii
              \labelwidth\leftmarginiii
              \advance\labelwidth-\labelsep
              \topsep    1\p@ \@plus 0.5\p@ \@minus 0.5\p@
              \parsep    \z@
              \partopsep 0.5\p@ \@plus 0\p@ \@minus 0.5\p@
              \itemsep \topsep}
\def\@listiv {\leftmargin\leftmarginiv
              \labelwidth\leftmarginiv
              \advance\labelwidth-\labelsep}
\def\@listv  {\leftmargin\leftmarginv
              \labelwidth\leftmarginv
              \advance\labelwidth-\labelsep}
\def\@listvi {\leftmargin\leftmarginvi
              \labelwidth\leftmarginvi
              \advance\labelwidth-\labelsep}

\providecommand{\maketitle}{}
\renewcommand{\maketitle}{%
  \par
  \begingroup
    \renewcommand{\thefootnote}{\fnsymbol{footnote}}
    \long\def\@makefntext##1{%
      \parindent 1em\noindent
      \hbox to 1.8em{\hss $\m@th ^{\@thefnmark}$}##1
    }
    \thispagestyle{empty}
    \@maketitle
    \@thanks
  \endgroup
  \let\maketitle\relax
  \let\thanks\relax
}

\newcommand{\@toptitlebar}{
  \hrule height 2\p@
  \vskip 0.25in
  \vskip -\parskip%
}
\newcommand{\@bottomtitlebar}{
  \vskip 0.29in
  \vskip -\parskip
  \hrule height 2\p@
  \vskip 0.09in%
}

\providecommand{\@maketitle}{}
\renewcommand{\@maketitle}{%
  \vbox{%
    \hsize\textwidth
    \linewidth\hsize
    \vskip 0.1in
    \@toptitlebar
    \centering
    {\LARGE\sc \@title\par}
    \@bottomtitlebar
    \textsc{\undertitle}\\
    \vskip 0.1in
    \def\And{%
      \end{tabular}\hfil\linebreak[0]\hfil%
      \begin{tabular}[t]{c}\bf\rule{\z@}{24\p@}\ignorespaces%
    }
    \def\AND{%
      \end{tabular}\hfil\linebreak[4]\hfil%
      \begin{tabular}[t]{c}\bf\rule{\z@}{24\p@}\ignorespaces%
    }
    \begin{tabular}[t]{c}\bf\rule{\z@}{24\p@}\@author\end{tabular}%
  \vskip 0.4in \@minus 0.1in \center{\@date}   \vskip 0.2in
  }
}

\newcommand{\ftype@noticebox}{8}
\newcommand{\@notice}{%
  \enlargethispage{2\baselineskip}%
  \@float{noticebox}[b]%
    \footnotesize\@noticestring%
  \end@float%
}

\renewenvironment{abstract}
{
  \centerline
  {\large \bfseries \scshape Abstract}
  \begin{quote}
}
{
  \end{quote}
}
\makeatother

\usepackage[utf8]{inputenc}
\usepackage[T1]{fontenc}
\usepackage{amsmath}
\usepackage{amssymb}
\usepackage{amsthm}
\usepackage{graphicx}
\usepackage{booktabs}
\usepackage{microtype}
\usepackage{placeins}
\usepackage{xcolor}
\usepackage{colortbl}
\usepackage{tabularx}
\usepackage{array}
\usepackage{multirow}
\usepackage{arydshln}
\usepackage{caption}
\usepackage[numbers]{natbib}
\usepackage[hidelinks]{hyperref}

\graphicspath{{figures/}}

\renewcommand{\topfraction}{0.95}      
\renewcommand{\bottomfraction}{0.70}
\renewcommand{\textfraction}{0.05}     
\newcommand{\appendixfloatregime}{%
  \raggedbottom
  \setlength{\parskip}{4pt plus 1pt}%
  \renewcommand{\topfraction}{0.99}%
  \renewcommand{\bottomfraction}{0.90}%
  \renewcommand{\textfraction}{0.01}%
  \renewcommand{\floatpagefraction}{0.70}%
  \setlength{\textfloatsep}{6pt plus 1pt minus 1pt}%
  \setlength{\floatsep}{6pt plus 1pt minus 1pt}%
}

\AddToHook{cmd/appendix/after}{\appendixfloatregime}

\makeatletter
\AtBeginDocument{\@ifpackageloaded{iclr2027_conference}{\setlength{\bibsep}{3pt plus 1pt minus 1pt}}{}}
\makeatother
\newcolumntype{L}{>{\raggedright\arraybackslash}X}   
\newcolumntype{Y}{>{\centering\arraybackslash}X}     

\definecolor{rowgraystrong}{gray}{0.78}
\definecolor{rowgraydark}{gray}{0.84}
\definecolor{rowgraymid}{gray}{0.89}
\definecolor{rowgraylight}{gray}{0.94}
\definecolor{rowgrayfaint}{gray}{0.98}

\newcommand{\tablemetrics}{%
  \setlength{\tabcolsep}{4pt}\renewcommand{\arraystretch}{1.2}}
\newcommand{\densetablemetrics}{%
  \setlength{\tabcolsep}{2pt}\renewcommand{\arraystretch}{1.1}}

\renewcommand{\shorttitle}{How Wrong Can a Good Predictor Be?}

\title{How Wrong Can a Good Predictor Be?\\Diverging Updates with Vanishing Predictive KL}

\author{%
  Qifu Wen\textsuperscript{1,2,*} \qquad Shuaijun Liu\textsuperscript{3,*} \\[0.35em]
  Zihan Zhou\textsuperscript{1} \qquad Xi Zeng\textsuperscript{1} \qquad Ningxin Su\textsuperscript{3} \\[0.55em]
  \textsuperscript{1}Boston University, Boston, MA, USA \\
  \textsuperscript{2}Shanghai Jiao Tong University, Shanghai, China \\
  \textsuperscript{3}The Hong Kong University of Science and Technology (Guangzhou), Guangzhou, China \\
  \textsuperscript{*}Co-first authors. \\
  \texttt{qfwen@bu.edu}
}

\begin{document}
\maketitle
\begin{abstract}
Accurate posterior prediction need not require accurate approximation of Bayesian updates. We prove that an unbounded gap between the update maps can coexist with vanishing predictive KL for every fixed finite $K\ge2$ in a stationary symmetric Gaussian HMM. Exact Bayesian mixing and an explicit deterministic radial filter act on the same $K-1$ belief coordinates. As $q\to0^+$, their separation in centered logits in the worst case grows at least linearly in the natural confidence scale $L_K(q)$, while their categorical $D_{\mathrm{KL}}(\mathrm{exact}\Vert\mathrm{radial})$ vanishes at the same explicit witness. Along stationary HMM trajectories, the expected terminal KL between filtered posteriors also converges to zero at $H(q)=\lceil-\log(q)/c\rceil+1$. Typical blocks without switches drive both filters into a common confidence cone, where softmax curvature suppresses their disagreement; a single Gaussian maximal event controls adaptive noise. A sweep with equally spaced Gaussians over $K\in\{2,4,8\}$ illustrates the opposing trends, and binary controls at long horizons compare saturating and nonsaturating recurrences. The result isolates two missing links between internal update gaps and predictive cost: the contribution of separating states to expected loss and decoder sensitivity. Thus even an unbounded internal update gap does not by itself certify predictive failure. The construction is fixed in $K$ and does not provide a universal criterion for when compression is harmless or characterize when internal gaps must incur task loss.
\end{abstract}
\section{Introduction}
\newtheorem{theorem}{Theorem}
\newtheorem{definition}[theorem]{Definition}
\newtheorem{proposition}[theorem]{Proposition}
\newtheorem{lemma}[theorem]{Lemma}
\newtheorem{remark}[theorem]{Remark}

Transformers retain a growing context, whereas state space and recurrent alternatives compress history into a deterministic state of fixed size. Recurrence expressivity and state tracking are therefore active design constraints \citep{mamba3_2026,aussm2025adaptiveUnitarySsm}, with formal comparisons also separating simplified selective SSMs from linear attention \citep{cohenkarlik2026expressivity}. But these architectural results do not answer the question at the task level: when does a strict computational mismatch force worse prediction? In a finite HMM, exact Bayes itself occupies only $K-1$ belief coordinates; the issue is whether a different update geometry with a fixed state must pay predictive loss.

Pointwise representability, learnability, and task cost weighted by the distribution are distinct. The predictive cost of a pointwise separation depends on the loss the decoder can see under the evaluation path law, including the loss magnitude on rare paths. Impossibility theorems at the architecture level establish representability gaps in specified models \citep{shakerinava2026diagonal,sarrof2024expressive}, whereas learnability can fail even for representable functions \citep{hahn2024sensitive}. Neither distinction alone determines this task cost. We study this boundary in a sequential model with finitely many states where both can be calculated.

Our result has two levels. At the level of the update maps, for every fixed finite $K\ge2$ an explicit moving witness makes the quotient distance in centered logits between exact Bayesian mixing and a radial filter grow as $\Omega(L_K(q))$, while their categorical $D_{\mathrm{KL}}(\mathrm{exact}\Vert\mathrm{radial})$ at that same witness tends to zero. At the level of the path law, with pairwise distinct scalar Gaussian means and slow symmetric switching, the filters consume the same stationary HMM observations and their expected terminal KL between filtered posteriors at $H(q)=\lceil-\log(q)/c\rceil+1$ also tends to zero. Thus even an unbounded internal separation need not induce a nonvanishing performance gap.

The result at the level of the maps is decoder geometry: both updates become increasingly confident in the same state, so softmax erases a growing logit discrepancy. The theorem under the path law adds a distributional argument. On a block with no latent switch, which occurs with high probability, one maximal event controls all adaptive Gaussian noise weights, radial saturation keeps the approximation in a confidence cone with exact Bayes, and a uniform KL envelope absorbs the remaining paths.

A sweep with equally spaced Gaussians over $K\in\{2,4,8\}$ illustrates the two opposing trends at finite scale: centered states separate while predictive KL shrinks. The binary case is a microscope rather than the theorem's scope: it visualizes the geometry and separates saturation from tanh specifically. The training experiments do not establish transfer to learned models.

We prove that, for fixed $K$, internal separation can diverge while predictive KL vanishes at the same input. We also prove convergence of expected predictive KL under the stationary law as the horizon grows in the limit of rare switching. The proof explains how decoder sensitivity and the contribution of separating states to expected loss control predictive error. Controlled numerical experiments illustrate the two results; simpler recurrences also achieve small KL at the tested short horizons (Table~\ref{tab:controls}). We do not claim uniformity in growing $K$, an impossibility result for a class of architectures, or learnability of the radial rule.

Figure~\ref{fig:general-k-overview} distinguishes the statements at the common input and under the path law.

\begin{figure*}[t]
\centering
\includegraphics[width=5.5in]{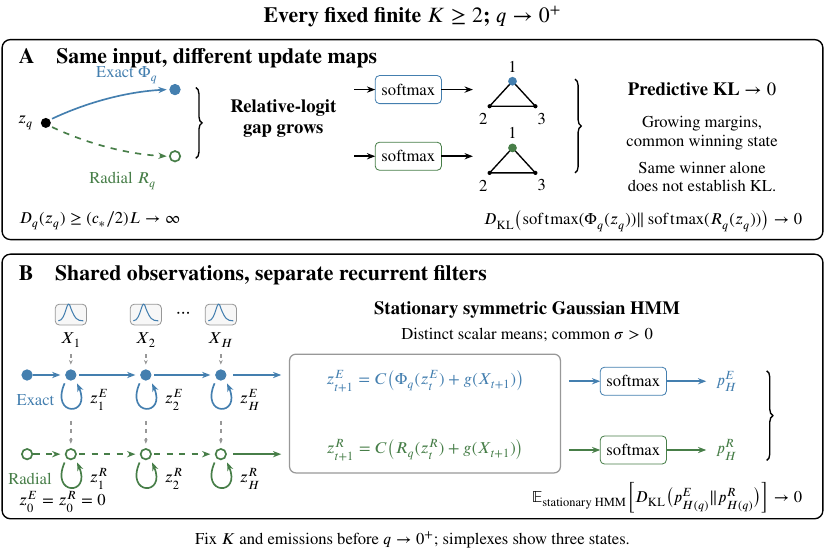}
\caption{\textbf{Two distinct fixed-$K$ separations.} (A) At a common moving input, exact and radial mixing have diverging separation in centered logits but vanishing categorical KL. (B) On shared stationary observations from the Gaussian HMM, their separate recurrences have vanishing expected terminal posterior KL at $H(q)$. Both statements fix $K$ and emission parameters before $q\to0^+$; the schematic does not assert that (A) implies (B).}
\label{fig:general-k-overview}
\end{figure*}
\section{Related Work}
\paragraph{Expressivity, learning, and task cost.}
Efficient sequence models make recurrence expressivity an explicit design variable. Recent work strengthens selective state space updates, analyzes their limits on tracking the state, and compares simplified selective layers with linear attention \citep{mamba3_2026,aussm2025adaptiveUnitarySsm,cohenkarlik2026expressivity,shakerinava2026diagonal,alsmann2026temporal}. Results on the worst case and on formal languages identify computations that specified architectures cannot represent \citep{merrill2024illusion,sarrof2024expressive}, while complementary results show that expressible functions can still be hard to learn \citep{hahn2024sensitive}. These results determine representation or optimization, not the predictive cost weighted by the distribution of a mismatch. Predictive cost additionally depends on decoder sensitivity and the contribution of the separating states weighted by loss under the evaluation path law. With unbounded KL, vanishing event probability alone does not control this contribution. An expressivity analysis for a specific task for GNNs makes a related diagnosis \citep{kemper2025what}; our contribution is an explicit sequential counterexample with a proved loss under the path law.

\paragraph{Harmless information loss and internal divergence.}
Compression that is aware of the task already shows that discarded information can be harmless relative to a downstream predictor family or distortion \citep{dubois2021lossy,dubois2020learning,hafezkolahi2021rate}. We do not optimize a representation or bit rate: the state dimension and prescribed update family are fixed, while their numerical maps depend on $q$ and their distance in centered logits diverges. The generic mathematical shape is not new in isolation. On separable data, gradient descent can send linear predictor norms to infinity while logistic loss tends to zero and the normalized direction converges \citep{soudry2018implicit}. Our conjunction is different: two explicit filtering maps diverge at the same centered input, their categorical KL vanishes there, and a second theorem proves predictive convergence under a stationary sequential law.

\paragraph{Approximate Bayesian filtering and finite memory.}
RNNs can approximate Bayesian filtering statistics on fixed finite horizons and, under additional conditions, uniformly in time \citep{bishop2023universal}. Neural Bayesian Filtering instead learns belief embeddings of fixed length and proves particle consistency over a finite horizon under its stated assumptions \citep{solinas2025neural}. Prediction with short memory gives a different benchmark: an order-$\ell$ Markov predictor attains average KL at most $I(\mathcal M)/\ell$ and average $\ell_1$ error at most $\sqrt{I(\mathcal{M})/\ell}$ \citep{sharan2018prediction}. For an $n$-state HMM, $I(\mathcal{M})\le\log n$, independently of mixing time. This benchmark concerns distributions of the next observation, whereas our theorem measures categorical filtered posteriors. Our predictor maintains $K-1$ belief coordinates. Streaming projection instead retains a truncated mixture over latent paths with a fixed budget through a deterministic recurrence \citep{duranmartin2026predictive}. The distinction is the maintained object and approximation operation, not recurrence itself; our result couples diverging centered map distance with vanishing categorical KL.

\paragraph{Filter stability and robustness.}
Classical filter stability asks whether filters with different initial conditions merge while running the correct model \citep{ocone1996asymptotic,chigansky2011intrinsic,mcdonald2020exponential}. Robustness theory instead perturbs model kernels and controls policy costs \citep{kara2020robustness,kara2022datadriven} or error in the filter kernel and policy performance \citep{demirci2025sensitivity} under the respective assumptions. Control policies with a finite window can also become nearly optimal as the window grows \citep{kara2021nearOptimality}. Our recurrence is neither a differently initialized exact filter nor a sequence of kernels converging to the true one; its pointwise update mismatch grows on the stated witness.

\paragraph{Detection delay and the binary controls.}
The affine control at the long endpoint is adjacent to quickest change detection. Page introduced the cumulative sum scheme \citep{page1954continuous}; Lorden introduced a criterion on conditional delay in the worst case and proved asymptotic optimality, Moustakides established exact CUSUM optimality under it, and Pollak studied nearly minimax detection under a different criterion on conditional delay \citep{lorden1971procedures,moustakides1986optimal,pollak1985optimal}. Pollak's irreversible model with a single change is adjacent to but not identical to our HMM that switches repeatedly. These works calibrate the binary diagnostic but do not imply the fixed-$K$ radial theorem.

\paragraph{Why state space models are the motivating case.}
The design pattern of a deterministic state spans structured SSMs, linear RNNs, gated linear attention, and recurrent competitors \citep{gu2020hippo,gu2022efficiently,smith2023simplified,orvieto2023resurrecting,yang2024gated,yang2024parallelizing,peng2023rwkv,beck2024xlstm,dao2024transformers}. Selective SSM theory and studies of state tracking show how transition structure, composition, precision, or normalization alter representational capability \citep{cirone2024selectivessm,siems2026state,bondaschi2025markov,structuredsparse2025,deltaproduct2025,li2025characterizing}. They motivate the logical question studied here. We do not prove that any named architecture implements, or fails to implement, either filter in our theorem.
\section{Problem Setup: Gaussian HMM with a Fixed Finite State and Radial Filter}
\begin{definition}[Stationary $K$-state Gaussian HMM and coupled filters]\label{def:setup}
Fix an integer $K\ge2$, pairwise distinct means $\mu_0,\ldots,\mu_{K-1}\in\mathbb R$, and $\sigma>0$, with $0<q<(K-1)/K$ so that $L_K(q)>0$ and $\alpha_K(q)>0$. Let $(S_t)_{t\ge0}$ be stationary and uniform on $\{0,\ldots,K-1\}$, with
\[
\Pr(S_{t+1}=i\mid S_t=i)=1-q,\qquad
\Pr(S_{t+1}=j\mid S_t=i)=\frac{q}{K-1}\quad(j\ne i).
\]
Conditional on $S_t=s$, let $X_{t+1}\sim\mathcal N(\mu_s,\sigma^2)$. Write $\mathcal C z=z-K^{-1}(\mathbf1^\top z)\mathbf1$ for centering and
\[
g(x)_s=\frac{\mu_s x}{\sigma^2}-\frac{\mu_s^2}{2\sigma^2}
\]
for the Gaussian score, understood modulo common shifts. For centered logits $z$, define exact transition mixing
\[
\Phi_q(z)=\mathcal C\log(P_q^\top\operatorname{softmax}(z)),
\]
where the logarithm is componentwise. Let
\[
r(z)=\sqrt2\,\|\mathcal C z\|_2,\qquad
\alpha_K(q)=1-\frac{Kq}{K-1},\qquad
L_K(q)=\log\frac{(K-1)(1-q)}{q},
\]
and define the radial map
\[
R_q(z)=\frac{L_K(q)\tanh(\alpha_K(q)r(z)/L_K(q))}{r(z)}\,\mathcal C z,
\]
with continuous value $R_q(0)=0$ and scale $\alpha_K(q)$ at the origin. Starting from $z^E_0=z^R_0=0$, the coupled filters use the same observations:
\[
z^E_{t+1}=\mathcal C\!\left(\Phi_q(z^E_t)+g(X_{t+1})\right),\qquad
z^R_{t+1}=\mathcal C\!\left(R_q(z^R_t)+g(X_{t+1})\right).
\]
Their terminal categorical outputs are $p^E_t=\operatorname{softmax}(z^E_t)$ and $p^R_t=\operatorname{softmax}(z^R_t)$. For $t\ge1$, $p^E_t$ is the filtered posterior of $S_{t-1}$ given $X_{1:t}$; $p^R_t$ approximates that posterior. These outputs precede any additional transition or emission channel.
\end{definition}

Define
\[
d_{\min}=\min_{i\ne j}\frac{(\mu_i-\mu_j)^2}{2\sigma^2}.
\]
Fix $0<c<d_{\min}$ and set $H(q)=\lceil(-\log q)/c\rceil+1$. Constants below may depend on the fixed $K$, the full mean vector, $\sigma$, and $c$. Appendix~\ref{app:geometry-guide} and Table~\ref{tab:geometry-dictionary} translate these centered coordinates into a geometric example with three states.
\section{Main Result and Proof Mechanism}
\begin{definition}[Centered separation between the update maps]\label{def:mismatch}
For centered $z\in\mathbb R^K$, define the quotient distance
\[
\mathcal D_q(z)=
\max_{i,j}\left|
[(\Phi_q(z))_i-(\Phi_q(z))_j]
-[(R_q(z))_i-(R_q(z))_j]
\right|.
\]
It ignores common logit shifts, which do not change a categorical prediction.
\end{definition}

\begin{theorem}[Diverging updates with vanishing predictive KL]\label{thm:quantitative-separation}
Fix any finite $K\ge2$, let $L=L_K(q)$, and set
\[
z_q=L(1/4,-1/4,0,\ldots,0),\qquad
c_\star=\frac12-\tanh\frac12>0.
\]
Then
\[
\frac{
[(\Phi_q(z_q))_0-(\Phi_q(z_q))_1]
-[(R_q(z_q))_0-(R_q(z_q))_1]}
{L}
\longrightarrow c_\star
\qquad(q\to0^+).
\]
Consequently, for all sufficiently small $q$, with the threshold allowed to depend on fixed $K$,
\[
\sup_{z:\mathcal Cz=z}\mathcal D_q(z)
\ge \mathcal D_q(z_q)
\ge \frac{c_\star}{2}L_K(q)
\longrightarrow\infty.
\]
Nevertheless, at the same witness,
\[
D_{\mathrm{KL}}\!\left(
\operatorname{softmax}\Phi_q(z_q)\,\middle\|\,
\operatorname{softmax}R_q(z_q)
\right)\longrightarrow0.
\]
\end{theorem}

\begin{theorem}[Predictive convergence with a fixed finite state]\label{thm:main}
\textbf{Setting.} For every fixed finite $K\ge2$, consider Definition~\ref{def:setup} with pairwise distinct means, $\sigma>0$, and $0<c<d_{\min}$. The two filters share the stationary HMM observation path and are evaluated at $H(q)=\lceil(-\log q)/c\rceil+1$.

\textbf{Conclusion.}
\[
\mathbb E\!\left[D_{\mathrm{KL}}\!\left(p^E_{H(q)}\,\middle\|\,p^R_{H(q)}\right)\right]\longrightarrow0
\qquad\text{as }q\to0^+.
\]
The expectation is under the stationary HMM path law.
\end{theorem}

\paragraph{Logical separation.}
Theorem~\ref{thm:quantitative-separation} is a statement about the maps in the worst case, with a matched decoder statement: the internal distance diverges at an explicit moving witness, yet the two decoded predictions agree there. Theorem~\ref{thm:main} is a distributional statement: predictive agreement also holds in expectation along trajectories generated by the stationary HMM. The first rules out the explanation that the update maps merely become uniformly close; the second shows that the phenomenon is not confined to one logit vector chosen by hand.

\paragraph{Mechanism.}
At the explicit witness, both maps place the same coordinate ahead by a margin proportional to $L_K(q)$, so categorical softmax curvature suppresses their growing internal discrepancy. Along HMM paths, with probability tending to one the logarithmic window contains no latent switch. Conditional on a fixed state, the centered score decomposes into a deterministic direction and one scalar Gaussian noise direction. Because every radial mixing coefficient lies in $(0,1)$, Abel summation controls the full adaptive noise process by one maximum of a Gaussian partial sum. For $K\ge3$, a radial barrier keeps the state at order $L_K(q)^{2/3}$ and a matching lower bound drives every margin of the true state over a wrong state to the same order after an initial transient. Exact Bayes has an even larger margin. The separate binary argument supplies a common logarithmic margin. The curvature bound in the common confidence cone makes the KL vanish on this event, while a $2L_K(q)$ envelope absorbs switch paths and failures of the maximal event. Appendix~\ref{app:proof-anatomy} explains these proof steps and the dependence of the constants on $K$.

\paragraph{Scope.}
Both results fix $K$ before taking $q\to0^+$. The quantitative witness moves outward with $L_K(q)$; it does not imply a gap bounded away from zero on every fixed bounded input. Neither theorem is uniform in growing $K$ or nearly colliding means, and neither is an impossibility theorem for an entire SSM architecture class or a claim about what training by gradient descent will learn. Table~\ref{tab:scope} separates proved statements, evidence at finite scale, and exclusions.

\begin{table}[t]
\caption{\textbf{Summary of results.} The first three rows are mathematical results. The fourth is an illustration at finite scale, and the last row records exclusions rather than negative results.}
\label{tab:scope}
\centering\small\tablemetrics
\begin{tabularx}{\linewidth}{>{\raggedright\arraybackslash}p{0.27\linewidth}>{\raggedright\arraybackslash}p{0.27\linewidth}X}
\toprule
\textbf{Object} & \textbf{Quantifier} & \textbf{Status} \\
\midrule
Centered map distance & every fixed finite $K\ge2$ & proved $\Omega(L_K)$ \\
\rowcolor{rowgraymid}
Categorical KL at the same input & every fixed finite $K\ge2$ & proved $\to0$ \\
Stationary terminal KL & every fixed finite $K\ge2$ & proved $\mathbb E\mathrm{KL}\to0$ \\
\rowcolor{rowgraymid}
Equally spaced Gaussian illustration & $K\in\{2,4,8\}$ & measured; not used in the proof \\
Growing $K$ or architecture class & $K=K(q)$; named models & not claimed \\
\bottomrule
\end{tabularx}

\end{table}
\section{Evidence at Finite Scale for Fixed K}
Theorem~\ref{thm:main} is asymptotic and does not specify a universal finite-$q$ onset. We therefore use a deliberately simple numerical illustration that matches its Gaussian assumptions. For $K\in\{2,4,8\}$, the state means are equally spaced, centered, and rescaled to minimum pairwise distance one. We evaluate $L\in\{8,16,32,64,128\}$ with 4,096 paired stationary paths per cell. Exact Bayes and the radial filter consume the same observations. The two displayed quantities are the centered distance between terminal states and the categorical $D_{\mathrm{KL}}(\mathrm{exact}\Vert\mathrm{radial})$ after the observation.

Figure~\ref{fig:fixed-k-bridge} shows increasing sampled state distance and decreasing categorical KL across all three displayed values of $K$: internal distance grows with the scale of rare switching while predictive KL falls. The curves are not used to establish the theorem or to claim a uniform onset in $K$. Appendix~\ref{app:empirical-replay} gives the estimands, coupling, uncertainty calculation, and exact summary statistics.

\begin{figure*}[t]
\centering
\includegraphics[width=0.49\textwidth]{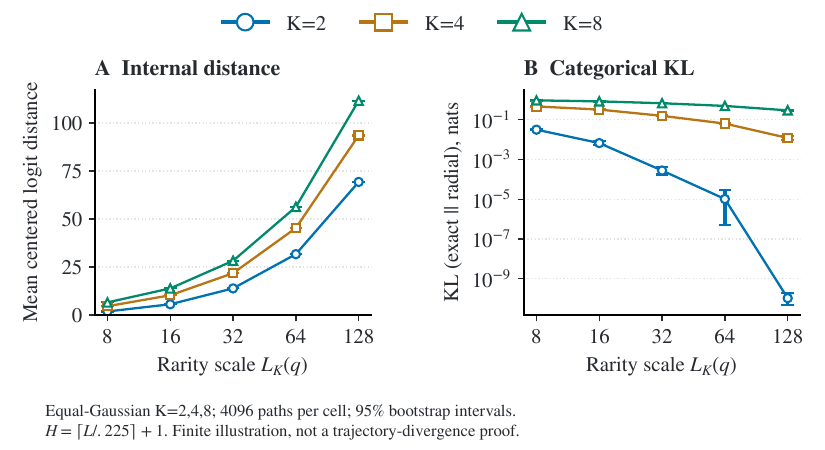}\hfill
\includegraphics[width=0.49\textwidth]{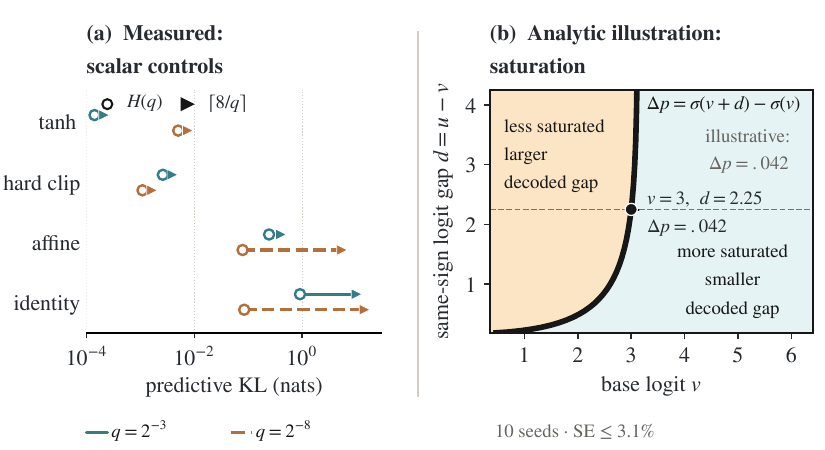}
\caption{\textbf{Separation at finite scale and decoder saturation.} Left pair: centered distance between terminal states grows while categorical $D_{\mathrm{KL}}(\mathrm{exact}\Vert\mathrm{radial})$ falls for Gaussian HMMs with $K=2,4,8$ (4,096 paired paths per cell; 95\% bootstrap intervals). Right pair: KL of the binary controls over ten seeds at $H(q)$ and $\lceil8/q\rceil$ (relative SE $\le3.1\%$), followed by an analytic illustration with a logistic decoder, not sampled trajectories. Nearly equal endpoints are displaced for legibility. Table~\ref{tab:controls} and Appendix~\ref{app:empirical-replay} give values, clipping conventions, and estimation details. Measurements illustrate, rather than prove, the asymptotic theorem.}
\label{fig:fixed-k-bridge}
\label{fig:mechanism}
\end{figure*}
\section{Binary Specialization and Geometry in the Finite Regime}
For $K=2$, translating the two means by their midpoint reduces the model to $(-\mu,+\mu)$ without changing likelihood ratios. Centered logits then collapse to one scalar $h$. Exact mixing becomes $F_q(h)$ and the radial witness becomes the tanh clamp $G_q(h)$ below. Theorem~\ref{thm:quantitative-separation} already supplies an explicit, diverging separation at the level of the maps for every fixed $K$. This binary slice instead serves as a visualization in the finite regime and isolates why an affine map cannot represent exact mixing; the stationary fixed-$K$ theorem does not depend on the affine obstruction proved here.

\begin{proposition}[Geometry of exact mixing]\label{prop:geometry}
For every $0<q<\tfrac12$ the exact mixing map $F_q$ of Definition~\ref{def:setup} satisfies:
\begin{enumerate}
\setlength{\itemsep}{0pt}\setlength{\parsep}{0pt}\setlength{\topsep}{2pt}
\item[(i)] \textbf{Saturation}: $|F_q(h)|<L(q)$ for every finite $h$.
\item[(ii)] \textbf{Contraction}: $F_q$ contracts logit perturbations with global factor at most $1-2q$, attained at the midpoint.
\item[(iii)] \textbf{Not affine}: with $x=\log 2$ and
\[
\Delta_q:=\bigl|F_q(0)-2F_q(x)+F_q(2x)\bigr|>0,
\]
every affine map $A$ satisfies $\max_{z\in\{0,x,2x\}}|F_q(z)-A(z)|\ge\Delta_q/4$.
\end{enumerate}
\end{proposition}

\emph{Proof.} See Appendix~\ref{app:proof}. \hfill$\square$

\begin{remark}[The obstruction does not give a lower bound]
Part (iii) of Proposition~\ref{prop:geometry} is adverse evidence for a naive pointwise argument. Direct expansion gives $\Delta_q=\tfrac34 q+O(q^2)$ as $q\to0^+$. Thus the obstruction is strict for every $q>0$ but its quantitative margin is not bounded away from zero along the family with rare switching, so it cannot yield a nonvanishing asymptotic predictive lower bound.
\end{remark}
Moreover, predictive loss is evaluated after the logistic decoder under the HMM path law, not in the logit norm in the worst case. The approximate map $G_q$ preserves the exact saturation scale and midpoint slope, and the shared envelope $|F_q(h)|,|G_q(h)|<L(q)$ gives global control without requiring uniform closeness.

The exact, radial, and affine transition geometry is shown in Figure~\ref{fig:binary-geometry} in the appendix.
\section{Binary Controls and Learned Transfer}
The experiment with learned models asks whether recurrences faithful to the architecture recover the proved mechanism across six switch probabilities. When trained end to end, the selective SSM arm closes at least half the reference gap at four of six $q$ values, below the prespecified criterion of five out of six; these results do not establish transfer to learned models. Under distillation, every nonreference comparison arm meets the threshold of half closure at all six values.

Closure is defined as $(\mathrm{KL}_{\mathrm{S6}}-\mathrm{KL}_{\mathrm{arm}})/(\mathrm{KL}_{\mathrm{S6}}-\mathrm{KL}_{\tanh})$, so the constrained scalar S6 arm sits at $0$, the analytic tanh recurrence sits at $1$, and values above $1$ mean the learned arm beats the proved mechanism. Here every architecture loss is $D_{\mathrm{KL}}(\text{exact Bayes}\Vert\text{model})$, evaluated after converting logits to float32 probabilities and clipping both arguments to $[\varepsilon_{32},1-\varepsilon_{32}]$, where $\varepsilon_{32}=\texttt{torch.finfo(torch.float32).eps}\approx1.1920929\times10^{-7}$. This differs from the metric for the scalar controls in Table~\ref{tab:controls}, which uses $10^{-15}$ clipping. Table~\ref{tab:perq} gives the breakdown at the horizon matched to the theorem $H(q)=\lceil-\ln q\rceil+1$.

\begin{table}[t]
\caption{\textbf{Closure per $q$ when trained end to end, mean $\pm$ standard error over ten
evaluation seeds.} Closure places the constrained scalar S6 arm at $0$ and the
analytic tanh recurrence at $1$, so the arm faithful to the architecture meets the criterion of half closure only at
$q\in\{2^{-5},2^{-6},2^{-7},2^{-8}\}$, the four empirically passing points. The asymptotic theorem supplies no finite-$q$ cutoff.}
\label{tab:perq}
\centering\small\tablemetrics
\begin{tabularx}{\linewidth}{l||*{3}{Y}}
\hline\hline
\rowcolor{gray!20}
\textbf{$q$} & \textbf{Mamba $\times2$} & \textbf{GRU} &
\textbf{Finite window} \\
\hline\hline
\rowcolor{gray!20}
\multicolumn{4}{l}{\emph{Below empirical half closure}} \\
$2^{-3}$ & $-1.59\pm0.39$ & $0.61\pm0.06$ & $-2.75\pm0.65$ \\
\rowcolor{gray!10}
$2^{-4}$ & $-0.09\pm0.21$ & $1.00\pm0.03$ & $-1.17\pm0.47$ \\
\hline
\rowcolor{gray!20}
\multicolumn{4}{l}{\emph{Above empirical half closure}} \\
$2^{-5}$ & $1.04\pm0.02$ & $1.11\pm0.03$ & $0.04\pm0.20$ \\
\rowcolor{gray!10}
$2^{-6}$ & $1.14\pm0.01$ & $1.18\pm0.01$ & $0.68\pm0.06$ \\
$2^{-7}$ & $1.13\pm0.02$ & $1.14\pm0.02$ & $0.74\pm0.07$ \\
\rowcolor{gray!10}
$2^{-8}$ & $1.16\pm0.01$ & $1.19\pm0.01$ & $0.87\pm0.03$ \\
\hline\hline
\end{tabularx}

\end{table}

The learned arm passes at the four smaller tested $q$ values, but this empirical split does not identify the theorem's asymptotic onset. At $q=2^{-3}$ the reference gap $\mathrm{KL}_{\mathrm{S6}}-\mathrm{KL}_{\tanh}$ is only $1.5\times10^{-3}$, so closure is scale sensitive and seed dispersion is widest at the two largest $q$. Even so, the two points below half closure lie $5.4$ and $2.8$ standard errors below the threshold, whereas the four passing points lie at least $26$ standard errors above it using the unrounded results across seeds. The outcome of four out of six is therefore not explained by marginal uncertainty. Under distillation, the closure computed from the ratio of mean losses ranges from $1.16$ to $3.97$ and every nonreference comparison arm passes at all six $q$ values. Table~\ref{tab:perq} instead reports the mean of closure ratios for each seed.

\paragraph{Generalization at long horizons separates the learned architectures.} The secondary endpoint on long paths evaluates the same trained models on sequences of length $\lceil 8/q\rceil$, up to $2048$ steps, against a training horizon matched to the theorem that is at most eight. Here the ranking inverts. The analytic tanh recurrence is horizon stable, its predictive KL staying between $1.7\times10^{-4}$ and $6.4\times10^{-3}$ at the two tested switch probabilities, while the scalar S6 arm degrades to $1.9\times10^{-1}$. The GRU inherits that stability, passing four of six $q$ when trained end to end and all six under distillation. The official Mamba blocks do not: when trained end to end, two stacked blocks pass only two of six, and a single block fails at every $q$, with a worst per-$q$ mean closure at long horizons of $-32.65$ at $q=2^{-3}$ when trained end to end. Under distillation, both Mamba depths pass zero of six at this long endpoint. We therefore do not claim that a canonical selective SSM acquires the mechanism. Matching the exact filter at the theorem's horizon does not ensure stability at longer horizons.

\paragraph{Scalar controls test a simpler explanation.} A naive recurrence could also achieve low loss at the horizon matched to the theorem, making the exhibited tanh geometry unnecessary. We therefore measured three controls against exact Bayes with the same generator ($\mu=\sigma=1$, 200{,}000 paths per point at $H(q)$, 20{,}000 at $\lceil 8/q\rceil$, split across ten fixed seeds). They are a plain affine recurrence $h\!\leftarrow\!(1-2q)h+\ell(x)$ with no saturation, the class considered by the pointwise lower bound; a hard clip of that map at $\pm L(q)$; and the identity accumulator. The third panel of Figure~\ref{fig:mechanism} shows the endpoint pattern; Appendix~\ref{app:empirical-replay}, Table~\ref{tab:controls}, reports the exact cells.

The controls support three conclusions. First, both saturating recurrences are horizon stable at the two tested switch probabilities, while the affine and identity controls degrade at the long endpoint. Their relative ordering is not uniform: the hard clip has lower predictive KL at $q=2^{-8}$, whereas the tanh clamp has lower predictive KL at $q=2^{-3}$. These two points do not identify a crossover threshold. The theorem should therefore be read as proving one saturating instance, not the optimal one. Second, the endpoint matched to the theorem does not establish a separation among scalar recurrence classes; the binary empirical separation among these scalar controls occurs at the long endpoint and is measured, not proved. At that endpoint the affine map reaches $5.53$ nats at $q=2^{-8}$, while both saturating controls remain stable. Third, the closure scale of Table~\ref{tab:perq} anchors $1$ to one saturating witness among several; the readings from the learned arm remain meaningful as calibrated distances to an analytic reference with a guarantee at the stated logarithmic horizon. Its stability at longer horizons is measured here, not established by that theorem.

Hyperparameters were selected by mean validation loss over three tuning seeds before confirmation, with ties broken deterministically and incomplete settings excluded. Tuning and confirmation seeds are disjoint, all final evaluation seeds are included, and the supplement gives the procedure for model selection and implementation details.
\section{Limitations}
Both theorems fix finite $K$ before taking $q\to0^+$. Their constants can deteriorate with $K$ and with nearly colliding means; no rate uniform in $K=K(q)$ is proved. The quantitative witness moves with $L_K(q)$, so it gives neither a gap on every bounded input nor a lower bound on the state dimension. The stationary theorem further assumes pairwise distinct scalar Gaussian means, symmetric switching, stationary initialization, and the stated logarithmic horizon.

The sweep with equally spaced Gaussians is an illustration at finite scale, not evidence for a uniform onset rate. The comparison concerns two explicit maps, not every SSM, and does not show that training learns the radial rule. Binary controls are diagnostics in the finite regime: hard clipping can beat tanh, and the training experiments do not establish transfer to learned models. The result is a counterexample, not a universal criterion for pruning, quantization, distillation, low rank adaptation, or other compression. The converse question, when an internal mismatch must incur nonvanishing task loss, remains open.
\section{Conclusion}
For every fixed finite latent state space, exact Bayesian mixing and our radial filter can diverge linearly in centered logits at an explicit witness while their decoded categorical KL vanishes. Predictive agreement also holds in stationary expectation in the symmetric scalar Gaussian HMM, so it is neither explained by uniform internal approximation nor confined to one input chosen by hand. A controlled sweep with equally spaced Gaussians over $K\in\{2,4,8\}$ illustrates both trends at finite scale, while binary controls compare behavior at short and long horizons. At the common input witness, softmax suppresses a growing map discrepancy; on typical HMM paths, a shared confidence cone suppresses predictive error. The result therefore identifies two missing links from internal separation to predictive loss: its contribution weighted by loss under the evaluation path law, and the decoder's sensitivity there. It does not show that this failure of gap transfer is prevalent in learned models or provide a converse criterion; uniform growing-$K$ rates and extensions to learned sequence models remain open.
\subsection*{Reproducibility statement}
The assumptions and complete paper proofs for the stated results are given in the appendices. Appendix~\ref{app:empirical-replay} specifies the empirical estimands, aggregation, model configurations, hyperparameters, seeds, and clipping conventions. The anonymous supplementary package contains the simulation code and numerical results used for the reported experiments.

\clearpage
\appendix
\section{Open Affine Control Lower Bound}
This appendix concerns a different question from Theorem~\ref{thm:main}: whether the nonsaturating binary affine control has a nonvanishing lower bound at the much longer endpoint $\lceil8/q\rceil$. The section on learned models establishes that separation only by measurement; the final analytic step for the joint event below remains open. The mechanism is \emph{sluggishness}: after a latent switch, the nonsaturating affine recurrence carries an unclamped state near its AR(1) fixed point $d/(2q)$, with $d=2\mu^2/\sigma^2$, and needs $\Theta(1/q)$ steps to change sign, while exact Bayes recovers on a shorter timescale for building confidence. Simulation localizes the cost accordingly: the fraction on the wrong side relative to the truth is between $0.18$ and $0.20$ across every $q$ tested. The subset on which Exact Bayes is confident is reported separately below. Disagreement with the exact filter was not measured. Table~\ref{tab:sluggish} separates the proved ingredients from the remaining step for the terminal joint event. For the concentration calculation, the centered oriented evidence has variance proxy $v=2d$, and at horizon $t$ we choose the deterministic prefix budget $B_t=td$.

\FloatBarrier

\begin{table}[htbp]
\caption{\textbf{The lower bound per step, step by step.} Every deterministic
and probabilistic ingredient is proved; the banded final row is the one step
still open, and the separation at long horizons is not claimed without it.}
\label{tab:sluggish}
\centering\small\tablemetrics
\begin{tabularx}{\linewidth}{>{\raggedright\arraybackslash}p{0.22\linewidth}L>{\raggedright\arraybackslash}p{0.14\linewidth}}
\toprule
\textbf{Step} & \textbf{Statement} & \textbf{Status} \\
\midrule
\rowcolor{rowgraymid}
\multicolumn{3}{l}{\emph{Deterministic core}} \\
Cost of a step on the wrong side & a confident exact predictive against a model on the wrong side
costs $\ge 3/80$ nats & proved \\
Affine state & $h_t=a^t h_0+\sum_{i<t}a^{t-1-i}e_i$ & proved \\
Summation by parts & $\lvert S_j\rvert\le B_t$ gives weighted sum $\ge -2B_t$ & proved \\
\rowcolor{rowgraymid}
Combining the bounds & no crossing by step $t$, and that step costs $\ge 3/80$ &
conditional on the margin \\
\hdashline
\rowcolor{rowgraymid}
\multicolumn{3}{l}{\emph{Probabilistic tail}} \\
Gaussian sub-Gaussianity & a centered real Gaussian is sub-Gaussian with constant
its variance & proved here \\
Prefix tail on a block & on a block with fixed latent signs, some prefix exceeds
$B_t$ with measure at most $2(t{+}1)e^{-B_t^2/(2tv)}$ & proved, discharges the tail \\
\midrule
\rowcolor{rowgraystrong}
\textbf{Terminal joint event} & \textbf{uniform probability that the terminal step is on the wrong side while exact Bayes is confident} & \textbf{open} \\
\bottomrule
\end{tabularx}

\end{table}

\FloatBarrier

With $v=2d$ and $B_t=td$, the charge from the prefix tail is $2(t{+}1)e^{-td/4}$. Instantiating the affine recurrence at $a=1-2q$ and $h_0=d/(2q)$ reduces the condition for no crossing $2B_t<a^t h_0$ to
\[
4qt<(1-2q)^t.
\]
Writing $y=qt$ and taking $q\to0$ gives the boundary equation $4y=e^{-2y}$, whose unique positive root is $y_\star=0.175866\ldots$. Hence $t^\ast q\to y_\star$, and stationarity gives limiting mass of the run age $1-e^{-y_\star}=0.161270\ldots$. Both constants are analytic, independent of $q$, and carry no sampling error. The latter is close to the measured fraction on the wrong side relative to the truth, $0.18$ to $0.20$ across $q=2^{-5}$ to $2^{-12}$ under the convention after mixing of Table~\ref{tab:perq}; the bound from summation by parts loses a factor of two, so the analytic constant need not match the observed fraction.

\paragraph{What the open step actually requires.} The estimand at the long endpoint is the predictive KL at the \emph{final step}, so no integration along a trajectory is needed: it suffices to lower bound, uniformly in $q$, the probability that the terminal step is a step on the wrong side, since the cost per step above then gives $\mathbb{E}[\mathrm{KL}_T]\ge(3/80)\,\mathbb{P}(W_{\mathrm{joint}})$. Stationarity supplies that bridge exactly. Writing $A_T$ for the age of the current latent run, memorylessness gives $\mathbb{P}(A_T\ge k)=(1-q)^k$, hence $\mathbb{P}(A_T\le t^\ast)\to1-e^{-0.176}$, which is the $0.161$ already reported. What remains open is therefore a uniform bound on the joint event that the run is young enough for the affine state not to have crossed and old enough for the exact filter to be confident, not an integration.

\paragraph{Why a stepwise union bound is insufficient.} Charging the worst case at every step forces the tolerance per step to grow like $\sqrt{\log(1/q)}$, shrinking the guaranteed horizon before crossing to $\Theta(1/(q\sqrt{\log(1/q)}))$ and making the resulting probability of the switch window tend to zero. This decay conflicts with the measured nearly constant fraction on the wrong side. Controlling the partial sum instead preserves a constant probability of the switch window.
\section{Geometry of Centered Logits and an Example with Three States}
\label{app:geometry-guide}

This section gives a geometric reading of Definition~\ref{def:setup}. It is not an additional assumption and is not used as a substitute for the proof. Its purpose is to make clear what the state dimension, centering operation, and radial nonlinearity mean when $K>2$.

\subsection{What $K$ counts}

The integer $K$ counts latent modes, not observations, layers, or time steps. A machine with normal, overheated, and failed modes has $K=3$ even if it is monitored for a million steps. At time $t$, a filter assigns a probability vector
\[
 p_t=(p_t(0),\ldots,p_t(K-1)),\qquad p_t(s)\ge0,\quad \sum_s p_t(s)=1.
\]
The vector has $K$ entries but only $K-1$ degrees of freedom because its entries sum to one. Logits make the same redundancy explicit: for every scalar $b$,
\[
 \operatorname{softmax}(z+b\mathbf 1)=\operatorname{softmax}(z).
\]
Thus $z$ and $z+b\mathbf 1$ represent the same belief. Centering chooses one representative from each equivalence class,
\[
 \mathcal Cz=z-\frac{\mathbf1^\top z}{K}\mathbf1,
 \qquad \mathbf1^\top\mathcal Cz=0.
\]
The meaningful filter state therefore lies in the $(K-1)$-dimensional hyperplane orthogonal to $\mathbf1$. Pairwise log odds $z_i-z_j$ are coordinates on this quotient space. This is why every mismatch in Definition~\ref{def:mismatch} is stated through pairwise differences rather than raw logits.

It is useful to keep two scale parameters separate. The state count $K$ fixes the dimension of the belief geometry. The quantity
\[
 L_K(q)=\log\frac{(K-1)(1-q)}q
\]
is a confidence scale determined by the switch probability. For fixed $K$, rare switching means $q\to0^+$ and hence $L_K(q)\to\infty$. The theorem takes this limit after fixing $K$; it does not let the number of modes grow with rarity.

\subsection{Direction records preference; radius records confidence}

Within the centered hyperplane, regard $z$ as an arrow. Its direction records the pattern of relative preferences among states. Its length records how strongly the filter prefers that pattern. Our convention $r(z)=\sqrt2\|\mathcal Cz\|_2$ makes this length agree with the absolute log odds in the binary specialization.

For a concrete $K=3$ belief, take
\[
 z=(2,0,-2),\qquad
 \operatorname{softmax}(z)\approx(0.867,0.117,0.016).
\]
State $0$ leads state $1$, which leads state $2$. If a radial operation halves the arrow, it produces $(1,0,-1)$ and probabilities approximately $(0.665,0.245,0.090)$. The ranking and all ratios between pairwise gaps are unchanged, but the confidence is lower. This example illustrates the geometry only; the actual radial factor is chosen by $q$ and by the current radius.

More precisely, the radial transition can be written
\[
 R_q(z)=\beta_q(z)\mathcal Cz,
 \qquad
 \beta_q(z)=\frac{L_K(q)\tanh(\alpha_K(q)r(z)/L_K(q))}{r(z)}.
\]
It preserves the direction of every nonzero centered state during the transition mixing step and changes only its radius. It is not an orthogonal projection: there is no fixed subspace onto which the state is dropped, and the multiplier depends nonlinearly on the current radius. Near the origin, $R_q(z)\approx\alpha_K(q)\mathcal Cz$; far from the origin, its radius saturates below $L_K(q)$. The next observation then adds the centered score $\mathcal Cg(x)$, which can change both direction and length. A trajectory can therefore turn many times even though each isolated radial mixing step preserves direction.

\subsection{What Exact Bayes does differently}

Exact transition mixing acts on probabilities before returning to centered logits. It adjusts the coordinates according to the full categorical mixture $P_q^\top p$, not through one shared radial multiplier. Except in special directions, it need not preserve the centered arrow's direction. The radial filter deliberately discards this geometry specific to each coordinate and retains only a common confidence contraction followed by the same observation score.

The two main theorems test different consequences of that difference. Theorem~\ref{thm:quantitative-separation} selects an explicit moving logit $z_q$ and proves that the exact and radial maps separate by order $L_K(q)$ in centered coordinates. This rules out the explanation that the two update maps merely converge to one another. Theorem~\ref{thm:main} then couples both filters to the same stationary observation path and proves that their decoded predictive distributions nevertheless converge in expected KL at the logarithmic horizon.

The apparent paradox is resolved by the softmax decoder. A centered displacement is expensive near a decision boundary, where probabilities are sensitive to logits. It can be cheap deep inside a shared confidence cone, where both filters put almost all mass on the same state. In the example with three states, moving from $(10,0,-10)$ to $(6,0,-6)$ is a large internal change, yet both decoded vectors are overwhelmingly concentrated on state $0$. The proof formalizes precisely this regime inside the common cone rather than asserting that every large logit difference is harmless.

\subsection{A dictionary for one step}

The following dictionary separates objects that are easy to conflate.

\begin{table}[t]
\caption{\textbf{Geometric dictionary for the fixed-$K$ result.} Observation scores are added after the transition map and may rotate the state.}
\label{tab:geometry-dictionary}
\centering\small\tablemetrics
\begin{tabularx}{\linewidth}{LX}
\toprule
\textbf{Object} & \textbf{Operational meaning} \\
\midrule
$K$ & Number of possible hidden modes; the centered belief state has dimension $K-1$. \\
\rowcolor{rowgraymid}
$q$ and $L_K(q)$ & Switch rarity and its associated confidence scale; $L_K(q)$ grows as switches become rarer. \\
$\mathcal Cz$ & Logit state after removing the common offset that does not affect prediction. \\
\rowcolor{rowgraymid}
$r(z)$ & Overall confidence magnitude, not a state count and not a time horizon. \\
$R_q$ & Nonlinear confidence cap that preserves centered direction during transition mixing; it is not a linear projection. \\
\bottomrule
\end{tabularx}

\end{table}

\section{Proof Roadmap for Fixed K}
This section is a reading guide to Appendix~\ref{app:proof}. The quantitative theorem at the level of the maps and the stationary theorem under the path law share an endpoint set by softmax curvature but answer different questions. The first follows from an explicit centered witness and direct asymptotics. The second has a probabilistic spine and two recovery arguments specific to the dimension.

\paragraph{Dependency order.} Theorem~\ref{thm:quantitative-separation} first compares the exact and radial gaps for the chosen pair at $z_q$, then invokes the curvature lemma for the common cone. For Theorem~\ref{thm:main} and $K\ge3$, Lemma~\ref{lem:abel} and the Gaussian maximal event feed the radial barriers and Lemma~\ref{lem:radial-recovery}; separately, the calculation over path mixtures yields Lemma~\ref{lem:exact-margin}. These two margins meet only in Lemma~\ref{lem:curvature}, with the deterministic envelope controlling the complement. The binary branch replaces the radial barrier step with common scalar recovery and then rejoins the same curvature and envelope argument.

\begin{figure*}[htbp]
\centering
\includegraphics[width=5.5in]{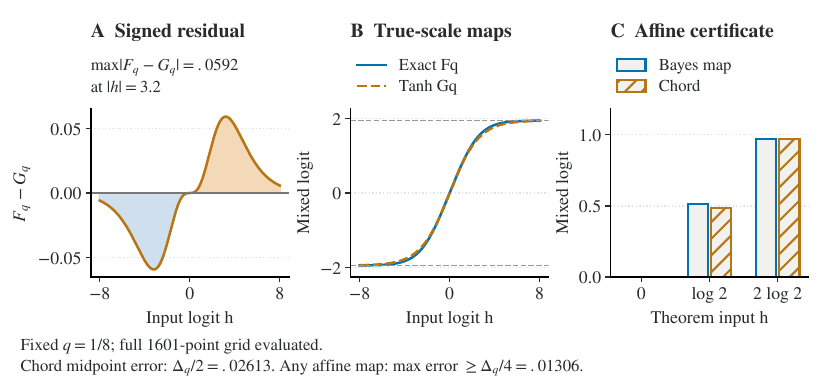}
\caption{\textbf{Binary transition geometry at $q=1/8$.} A: signed residual of exact minus tanh. B: the maps at true scale share saturation levels and midpoint slope. C: values at $0,\log2,2\log2$ certify that the map is not affine; their second difference $\Delta_q$ forces every affine approximation to incur maximum error at least $\Delta_q/4$. This is a fixed-$q$, $K=2$ diagnostic.}
\label{fig:binary-geometry}
\end{figure*}

\subsection{Proof outline}

Both filters start in the same centered belief coordinates and consume the same observations. The proof under the path law first conditions on a constant latent block. It then shows that exact Bayes and the radial recurrence identify the true state with diverging margins. A deterministic lemma on the categorical KL converts that common confidence cone into an exponentially small loss on the good event, while a $2L_K(q)$ envelope controls every remaining path.

\subsection{The $K\ge3$ branch}

On a block with no switch in state $s$, the centered Gaussian score is $a_s+wZ_t$. Because radial mixing uses one adaptive scalar coefficient, the approximate state has the exact representation $z^R_t=U_ta_s+V_tw$. Abel summation bounds every $V_t$ by twice one maximum of a Gaussian partial sum, so a single maximal inequality controls the whole horizon without a temporal union bound. Projection onto $w^\perp$ supplies a deterministic signal direction; upper and lower radial barriers then give margins for the true state of order $L_K(q)^{2/3}$. Recovery of the exact filter is established separately by comparing the constant path in the true state with constant wrong paths and the aggregate Bayes factor of all alternatives that contain a switch.

\subsection{The binary branch}

For $K=2$, the signal and noise directions are collinear, so the projection step is unavailable. Midpoint translation reduces arbitrary distinct means to the symmetric model. An event on the prefix evidence forces both scalar recurrences to hit a common margin, and a geometric sum over possible hit locations controls suffix persistence. Figure~\ref{fig:binary-geometry} visualizes the transition maps used by this branch. The argument unions suffix lengths, not transition errors at one step, and its failure probability remains negligible after multiplication by the KL envelope.

\subsection{Conditioning argument}

There are three changes of probability law. First, the event that no switch occurs costs at most $H(q)q$. Second, conditional on that event, Gaussian concentration applies to independent emissions from the fixed latent state. Third, the exact posterior still sums over every latent path; candidate paths that contain a switch are included and controlled through their aggregate Bayes factor. The final expectation is recovered by averaging the finitely many conditional bounds over the stationary uniform starting state.
\section{Proof Anatomy, Constants, and Quantifiers}
\label{app:proof-anatomy}

This section explains the main steps in the proof of Theorem~\ref{thm:main} and the dependence of its constants on the model parameters. Full proofs are given in Appendix~\ref{app:proof}.

\subsection{Quantifier order}

Fix $K$, the distinct means $\mu_0,\ldots,\mu_{K-1}$, the variance $\sigma^2$, and a horizon constant $c<d_{\min}$, where
\[
 d_{\min}=\min_{i\ne j}\frac{(\mu_i-\mu_j)^2}{2\sigma^2}>0.
\]
Only then send $q\to0^+$. Constants denoted $C_K,c_s,\kappa$, or a small-$q$ threshold may depend on these fixed parameters. This permits a finite union over states and competitors, but it does not provide a bound uniform in $K$ or in a family whose means approach one another. The theorem has the logical form
\[
 \forall K<\infty\;\forall(\mu_0,\ldots,\mu_{K-1})\text{ distinct},
 \qquad \lim_{q\to0^+}\mathbb E D_{\rm KL}(p_H^E\Vert p_H^R)=0,
\]
not a joint limit over $K$ and $q$.

\subsection{The probabilistic argument for $K\ge3$}

Condition on a block with no switch whose true state is $s$. The centered Gaussian score has the exact form with two vectors
\[
 \mathcal Cg(X_{t+1})=a_s+wZ_{t+1},
\]
so the radial recurrence stays in the random plane spanned by $a_s$ and $w$:
\[
 z_t^R=U_ta_s+V_tw,\qquad
 U_{t+1}=\beta_tU_t+1,\quad
 V_{t+1}=\beta_tV_t+Z_{t+1}.
\]
This does not reduce the filter to two states. It reduces the analysis of one conditioned block to one accumulated signal coefficient $U_t$ and one accumulated noise coefficient $V_t$.

The multiplier $\beta_t$ depends on the past, so the noise is not a deterministic weighted Gaussian sum. The key observation is order, not independence: for fixed terminal $t$, the products of positive multipliers attached to later observations are nondecreasing. Abel summation therefore gives the pathwise inequality
\[
 |V_t|\le2\max_{u\le H}\left|\sum_{k=1}^u Z_k\right|
\]
simultaneously for all $t$. One maximal bound from an exponential martingale controls this maximum. The proof does not apply a separate tail bound at every time and then union bound over the horizon.

On the resulting event, $|V_t|=O(\sqrt{L\log L})$. The tanh expansion supplies an upper barrier $U_t=O(L^{2/3})$ and a matching lower bound after an initial transient of order $L^{2/3}$. The exponent $2/3$ is the balance at which the unit signal increment competes with the cubic correction in $\tanh x=x-\Theta(x^3)$. It is sufficient for recovery because
\[
 z^R_{t,s}-z^R_{t,j}
 =d_{sj}U_t+\eta_{sj}V_t,\qquad
 d_{sj}=\frac{(\mu_s-\mu_j)^2}{2\sigma^2}>0,\quad
 \eta_{sj}=\frac{\mu_s-\mu_j}{\sigma},
\]
and $L^{2/3}$ dominates $\sqrt{L\log L}$.

Exact Bayes is handled separately. The constant-$s$ latent path is the likelihood reference. Constant wrong paths lose exponentially in $L$, while the aggregate Bayes factor of all switched paths has expectation of order $Hq$. This is a calculation over path mixtures; it does not assume independence between posterior coordinates. Because $H=\Theta(L)$ and $q=\Theta(e^{-L})$ for fixed $K$, mass of the switched paths vanishes.

\subsection{Bounding predictive KL}

The radial and exact arguments meet only after both filters have entered a common confidence cone. On the good event, every margin of the true state over a wrong state is at least $c_*L^{2/3}$. Along the line segment between the two terminal logits, softmax curvature is then exponentially small. Combining this curvature with the deterministic terminal range bound yields
\[
 D_{\rm KL}(p_H^E\Vert p_H^R)
 \le C_KL^2e^{-c_*L^{2/3}}.
\]
The bound on this event alone is not enough because KL is unbounded in arbitrary logits. Lemma~\ref{lem:envelope} supplies a pathwise $2L$ envelope, which is multiplied by the vanishing probability of the complement. Combining the bound on this event with the uniform bound on its complement turns recovery with high probability into convergence of expected KL.

\subsection{Why the binary branch is separate}

For $K\ge3$, distinct scalar means imply that $a_s$ cannot be collinear with the common noise direction $w$: otherwise a nonzero quadratic would agree with an affine function at at least three distinct points. This makes the radial radius control available. With two points, that polynomial contradiction disappears. The binary proof instead translates arbitrary means to the symmetric pair $(-\delta,+\delta)$ and controls both scalar recurrences through a common recovery event. It then rejoins the same KL curvature and envelope lemmas. The separate branch is an issue of proof geometry, not a restriction to symmetric models with two states.

\subsection{Rates proved and rates not claimed}

The witness at the level of the maps and stationary path result have different rates and should not be merged. At the explicit witness, centered map separation is asymptotically $c_\star L$, and decoder KL is bounded by a polynomial in $L$ times $e^{-L/5}$. Along typical blocks with no switch in the stationary proof, the radial margin for the true state is only required to grow as $L^{2/3}$, which still makes the decoder cost on the good event vanish. Neither statement identifies an optimal exponent, a sharp finite-$q$ onset, or constants uniform over state geometries. The curves with equally spaced Gaussians in Appendix~\ref{app:empirical-replay} are therefore illustrations rather than rate estimates.

\section{Proofs of the Main Results}
\label{app:proof}

This appendix proves Theorems~\ref{thm:quantitative-separation}
and~\ref{thm:main} for every fixed finite $K\ge2$. The quantitative
theorem at the level of the maps uses an explicit centered witness and a direct asymptotic
calculation. The stationary theorem uses a representation in two scalars of the
radial filter, one Gaussian maximal event, and an analysis over path mixtures of
exact Bayes for $K\ge3$; the binary case uses the same KL bounds after
midpoint translation and a recovery argument in one dimension.

\subsection{Preliminaries}

Write $T=-\log q$, $\alpha=\alpha_K(q)$, and $L=L_K(q)$.  For fixed $K$,
\[
 L=T+\log(K-1)+\log(1-q)=T+O(1).
\]
Consequently, for all sufficiently small $q$, $L>1$, $\alpha\ge1/2$, and
$H\le C_HL$.  We will also use
\[
 q=\frac{K-1}{e^L+K-1}\le(K-1)e^{-L},\qquad
 Hq\longrightarrow0,\qquad (1-q)^{-(H-1)}\longrightarrow1.
\tag{10}
\]

Fix a state $s$ and let
\[
 \mathcal N_s=\{S_0=\cdots=S_{H-1}=s\}.
\]
Conditional on $\mathcal N_s$,
$X_{t+1}=\mu_s+\sigma Z_{t+1}$ for $t=0,\ldots,H-1$, where the
$Z_t$ are independent standard Gaussians.  With
\[
 a_s=\mathcal C\left(
       \frac{\mu_s\boldsymbol\mu}{\sigma^2}
       -\frac{\boldsymbol\mu^{\odot2}}{2\sigma^2}\right),
 \qquad
 w=\mathcal C\left(\frac{\boldsymbol\mu}{\sigma}\right),
\tag{11}
\]
the centered score is $g(X_{t+1})=a_s+wZ_{t+1}$.

\begin{lemma}[Radial dynamics in two scalars]\label{lem:two-scalar}
On $\mathcal N_s$, adapted scalar processes $U_t,V_t$ satisfy
\[
 z^R_t=U_ta_s+V_tw,\qquad U_0=V_0=0,
\]
\[
 U_{t+1}=\beta_tU_t+1,\qquad
 V_{t+1}=\beta_tV_t+Z_{t+1},
\tag{12}
\]
where
\[
 \beta_t=\frac{L\tanh(\alpha r(z^R_t)/L)}{r(z^R_t)}
\]
with continuous value $\alpha$ at the origin.  Moreover,
$0<\beta_t\le\alpha<1$ and $U_t\ge0$.
\end{lemma}

\par\noindent\emph{Proof.}
The radial transition multiplies every centered coordinate by the same
scalar $\beta_t$.  Substitution of (11) proves (12) by induction.
Positivity is immediate, while $\tanh x\le x$ gives
$\beta_t\le\alpha$.
\hfill$\square$\par

\subsection{One maximal event controls the adaptive noise}

Let $S_u=\sum_{k=1}^u Z_k$.

\begin{lemma}[Adaptive Abel bound]\label{lem:abel}
For every realized path and every $t\le H$,
\[
 |V_t|\le2\max_{u\le H}|S_u|.
\tag{13}
\]
\end{lemma}

\par\noindent\emph{Proof.}
Expanding (12) gives
\[
 V_t=\sum_{k=1}^t\gamma_{k,t}Z_k,\qquad
 \gamma_{k,t}=\prod_{\ell=k}^{t-1}\beta_\ell,\qquad
 \gamma_{t,t}=1.
\]
Although the weights are adaptive, they obey
$0\le\gamma_{1,t}\le\cdots\le\gamma_{t,t}=1$ pathwise.  Abel summation
therefore gives
\[
 V_t=S_t-\sum_{k=1}^{t-1}
       (\gamma_{k+1,t}-\gamma_{k,t})S_k.
\]
The coefficients in the second term are nonnegative and sum to at most
one, proving (13). This pathwise bound also applies to adaptive weights.
\hfill$\square$\par

\begin{lemma}[Gaussian maximal event]
Set
\[
 x_L=\sqrt{2H\log(2L^3)},\qquad
 \mathcal E_L=\left\{\max_{u\le H}|S_u|\le x_L\right\}.
\]
Then $\Pr(\mathcal E_L^c)\le L^{-3}$.  On $\mathcal E_L$,
\[
 \max_{t\le H}|V_t|\le M_L:=2x_L
 =O(\sqrt{L\log L}).
\tag{14}
\]
\end{lemma}

\par\noindent\emph{Proof.}
For every $\lambda>0$,
$\exp(\lambda S_u-u\lambda^2/2)$ is a nonnegative martingale.  Doob's
maximal inequality, optimized at $\lambda=x/H$, yields
\[
 \Pr\!\left(\max_{u\le H}S_u>x\right)
 \le e^{-x^2/(2H)}.
\]
Apply the same argument to $-S_u$ and union only the two signs.  Substituting
$x=x_L$ and applying Lemma~\ref{lem:abel} proves the claim simultaneously over the full horizon.
\hfill$\square$\par

\subsection{Radial recovery when $K\ge3$}

Assume $K\ge3$, and let $P$ be Euclidean projection onto $w^\perp$.

\begin{lemma}[Noncollinear signal]
For every state $s$, $Pa_s\ne0$.
\end{lemma}

\par\noindent\emph{Proof.}
Distinct means imply $w\ne0$.  If $Pa_s=0$, centeredness would give
$a_s=\lambda w$.  Hence a constant $b$ would satisfy
\[
 \frac{\mu_s\mu_i}{\sigma^2}-\frac{\mu_i^2}{2\sigma^2}
 =\frac{\lambda\mu_i}{\sigma}+b
\]
for every $i$.  After clearing denominators, a quadratic with nonzero
quadratic coefficient would vanish at the $K\ge3$ distinct values
$\mu_i$, a contradiction.
\hfill$\square$\par

Set $D_s=\sqrt2\|Pa_s\|_2>0$, $A_s=r(a_s)$, and $W=r(w)$.  From
Lemma~\ref{lem:two-scalar},
\[
 D_sU_t\le r(z^R_t)\le A_sU_t+W|V_t|.
\tag{15}
\]

\begin{lemma}[Upper radial barrier]\label{lem:upper-barrier}
For each $s$, constants $C_U,C_r,q_s>0$ exist such that, on
$\mathcal E_L$ and for $q<q_s$,
\[
 U_t\le C_UL^{2/3},\qquad r(z^R_t)\le C_rL^{2/3}
\tag{16}
\]
for all $t\le H$.
\end{lemma}

\par\noindent\emph{Proof.}
The function $h(x)=\tanh(x)/x$, continuously extended at zero, decreases on
$[0,\infty)$.  Using the lower bound in (15),
\[
 \beta_tU_t
 =\alpha h(\alpha r(z^R_t)/L)U_t
 \le\frac{L}{D_s}\tanh(\alpha D_sU_t/L).
\]
Let $f_L(u)=1+(L/D_s)\tanh(\alpha D_su/L)$.  Choose $C_U$ so that
$D_s^2C_U^3>64$, and put $R=C_UL^{2/3}$.  For large $L$,
$\alpha D_sR/L\le1$.  Since
$\tanh x\le x-x^3/8$ on $[0,1]$,
\[
 f_L(R)-R
 \le1-(1-\alpha)R-\frac{\alpha^3D_s^2C_U^3}{8}
 \le1-\frac{D_s^2C_U^3}{64}<0.
\]
The map $f_L$ is increasing, $U_0\le R$, and
$U_{t+1}\le f_L(U_t)$, so induction on the first crossing yields $U_t\le R$.
The upper bound in (15), together with (14) and
$M_L=o(L^{2/3})$, gives the radius bound.
\hfill$\square$\par

\begin{lemma}[Lower memory after the initial transient]
There are constants $c_U,C_0,q_s'>0$ such that, on $\mathcal E_L$ and
for $q<q_s'$,
\[
 U_t\ge c_UL^{2/3}
\tag{17}
\]
whenever $C_0L^{2/3}\le t\le H$.
\end{lemma}

\par\noindent\emph{Proof.}
The global inequality $\tanh x\ge x-x^3/3$ and
Lemma~\ref{lem:upper-barrier} imply
\[
 \beta_t\ge
 \alpha-\frac{\alpha^3r(z^R_t)^2}{3L^2}
 \ge1-\varepsilon_L,\qquad
 \varepsilon_L=\frac{Kq}{K-1}+\frac{C_r^2}{3L^{2/3}}.
\tag{18}
\]
For small $q$, $0<\varepsilon_L<1$ and
$b_1L^{-2/3}\le\varepsilon_L\le b_2L^{-2/3}$ for positive constants
$b_1,b_2$.  Iterating (12) gives
\[
 U_t\ge\frac{1-(1-\varepsilon_L)^t}{\varepsilon_L}.
\]
For $t\ge\lceil(\log2)/\varepsilon_L\rceil$, the numerator is at least
$1/2$, which proves (17).  This transient is $O(L^{2/3})$ and is smaller
than $H=\Theta(L)$.
\hfill$\square$\par

\begin{lemma}[Simultaneous radial recovery]\label{lem:radial-recovery}
For every $s$, constants $c_s,q_s''>0$ exist such that, on
$\mathcal E_L$ and for $q<q_s''$,
\[
 z^R_{t,s}-z^R_{t,j}\ge c_sL^{2/3}
\tag{19}
\]
for every $j\ne s$ and $C_0L^{2/3}\le t\le H$.
\end{lemma}

\par\noindent\emph{Proof.}
Direct calculation gives
\[
 (a_s)_s-(a_s)_j
 =d_{sj}:=\frac{(\mu_s-\mu_j)^2}{2\sigma^2}>0,\qquad
 w_s-w_j=\eta_{sj}:=\frac{\mu_s-\mu_j}{\sigma}.
\]
Thus
$z^R_{t,s}-z^R_{t,j}=d_{sj}U_t+\eta_{sj}V_t$.  The first term is
bounded below by a positive multiple of $L^{2/3}$, uniformly over the
finitely many competitors, whereas the second is
$O(\sqrt{L\log L})=o(L^{2/3})$ by (14).
\hfill$\square$\par

\subsection{Recovery of the exact filter by path mixtures}

Continue to condition on $\mathcal N_s$.  The exact posterior after $H$
observations is the normalized sum of prior times likelihood over latent paths.
Use the constant path $(s,\ldots,s)$ as reference; its prior mass is
\[
 \pi_*=\frac1K(1-q)^{H-1}.
\tag{20}
\]

\begin{lemma}[Constant wrong paths]\label{lem:constant-wrong}
There are $\kappa,q_s>0$ such that, for $q<q_s$, with conditional
probability at least $1-(K-1)e^{-\kappa L}$, every constant path
$j\ne s$ has likelihood ratio relative to the reference at most
$e^{-L/2}$.
\end{lemma}

\par\noindent\emph{Proof.}
For $j\ne s$, let
\[
 Y_t^{s,j}=\log
 \frac{\varphi_{\mu_s,\sigma}(X_t)}
      {\varphi_{\mu_j,\sigma}(X_t)}.
\]
Under $\mathcal N_s$, these variables are independent Gaussians with mean
$d_{sj}$ and variance $2d_{sj}$.  Since
$Hd_{sj}\ge(d_{\min}/c)T$ with $d_{\min}/c>1$, while $L=T+O(1)$, a
bound on the Gaussian lower tail gives
\[
 \Pr\!\left(\sum_{t=1}^H Y_t^{s,j}<L/2\right)
 \le e^{-\kappa_{sj}L}
\]
for small $q$.  A finite union over $j\ne s$ proves the claim.
\hfill$\square$\par

\begin{lemma}[Bayes factor of a switched path]\label{lem:switch-bayes}
Let $\mathcal P_{\rm sw}$ be the latent paths containing at least one switch,
let $\pi(p)$ denote their prior masses, and let $\mathrm{LR}_p$ be likelihood
relative to the constant-$s$ path.  For
\[
 B_{\rm sw}=
 \frac{\sum_{p\in\mathcal P_{\rm sw}}\pi(p)\mathrm{LR}_p}{\pi_*},
\]
\[
 \mathbb E[B_{\rm sw}\mid\mathcal N_s]
 \le\frac{KHq}{(1-q)^{H-1}},
\qquad
 \Pr(B_{\rm sw}>e^{-L/2}\mid\mathcal N_s)
 \le CLe^{-L/2}.
\tag{21}
\]
\end{lemma}

\par\noindent\emph{Proof.}
Under the reference observation law, every likelihood ratio has expectation
one.  The numerator in the first expectation is therefore the prior mass of
paths that contain a switch, at most $Hq$.  Division by (20) gives the first
bound.  Markov's inequality at threshold $e^{-L/2}$, followed by (10),
gives the second.
\hfill$\square$\par

\begin{lemma}[Exact posterior margin]\label{lem:exact-margin}
Conditional on $\mathcal N_s$, outside an event of probability at most
\[
 (K-1)e^{-\kappa L}+CLe^{-L/2},
\tag{22}
\]
the exact terminal logits satisfy
\[
 z^E_{H,s}-z^E_{H,j}\ge L/2-\log K
\tag{23}
\]
for every $j\ne s$.
\end{lemma}

\par\noindent\emph{Proof.}
On the intersection of the events in
Lemmas~\ref{lem:constant-wrong} and~\ref{lem:switch-bayes}, total
nonreference weight divided by reference weight is at most
$Ke^{-L/2}$.  Hence
$p^E_H(s)/p^E_H(j)\ge e^{L/2}/K$ for every $j\ne s$, and taking
logarithms gives (23).
\hfill$\square$\par

\subsection{Categorical KL bounds}

For vectors $u,v$, put
$\Delta=u-v$ and
$\operatorname{range}(\Delta)=\max_i\Delta_i-\min_i\Delta_i$.

\begin{lemma}[Range envelope]
For every $u,v\in\mathbb R^K$,
\[
 D_{\rm KL}(\operatorname{softmax}u\|
             \operatorname{softmax}v)
 \le\operatorname{range}(u-v).
\tag{24}
\]
\end{lemma}

\par\noindent\emph{Proof.}
Let $A(z)=\log\sum_i e^{z_i}$.  The increment $A(u)-A(v)$ lies between
$\min_i\Delta_i$ and $\max_i\Delta_i$.  Thus every log probability ratio
is at most $\operatorname{range}(\Delta)$, and averaging under
$\operatorname{softmax}(u)$ proves (24).
\hfill$\square$\par

\begin{lemma}[Terminal envelope]\label{lem:envelope}
For every observation path,
\[
 D_{\rm KL}(\operatorname{softmax}z^E_H\|
             \operatorname{softmax}z^R_H)\le2L.
\tag{25}
\]
\end{lemma}

\par\noindent\emph{Proof.}
Let $m^E=\Phi_q(z^E_{H-1})$ and
$m^R=R_q(z^R_{H-1})$ be the logits before the observation at the terminal
step.  The final shared score cancels, so
$z^E_H-z^R_H=m^E-m^R$.  Exact mixing places every probability
coordinate in $[q/(K-1),1-q]$, hence every exact pairwise score before the observation
difference has magnitude at most $L$.  The radial image has pairwise
diameter strictly below $L$ because its radial norm is below $L$.
Therefore $\operatorname{range}(m^E-m^R)\le2L$, and (24) applies.
\hfill$\square$\par

\begin{lemma}[Curvature on a common cone]\label{lem:curvature}
Suppose a state $s$ and $m\ge0$ satisfy
\[
 u_s-u_j\ge m,\qquad v_s-v_j\ge m
\]
for every $j\ne s$, and $\operatorname{range}(u-v)\le R$.  Then
\[
 D_{\rm KL}(\operatorname{softmax}u\|
             \operatorname{softmax}v)
 \le K(K-1)R^2e^{-m}.
\tag{26}
\]
\end{lemma}

\par\noindent\emph{Proof.}
Every point $u_\tau=(1-\tau)u+\tau v$ has the same margin lower bound.
If $p_\tau=\operatorname{softmax}(u_\tau)$, then
$1-p_\tau(s)\le(K-1)e^{-m}$.  The Hessian of log-sum-exp is
$\operatorname{diag}(p_\tau)-p_\tau p_\tau^\top$, so its operator norm
is at most $2(K-1)e^{-m}$.  The Hessian annihilates common shifts, and
$\operatorname{range}(u-v)\le R$ implies
$\|\mathcal C(u-v)\|_2^2\le KR^2$.  The integral Taylor formula for
the corresponding Bregman divergence now gives (26).
\hfill$\square$\par

\subsection{Binary specialization}

When $K=2$, let
\[
 m_0=\frac{\mu_0+\mu_1}{2},\qquad
 \delta=\frac{|\mu_1-\mu_0|}{2}.
\]
After relabeling if needed and translating observations by $m_0$, the
emission means are $-\delta,+\delta$.  This preserves the likelihood ratio,
the latent law, and both recurrences for the pairwise logits.  Their evidence drift
and variance are
\[
 d=\frac{2\delta^2}{\sigma^2}
   =\frac{(\mu_1-\mu_0)^2}{2\sigma^2}=d_{\min},
 \qquad \operatorname{Var}(Y_t)=2d.
\tag{27}
\]
In pairwise log odds the two transition maps are
\[
 F_q(h)=2\operatorname{artanh}\!\left((1-2q)\tanh(h/2)\right),
 \qquad
 G_q(h)=L\tanh((1-2q)h/L).
\tag{28}
\]

\begin{lemma}[Binary common recovery]\label{lem:binary-recovery}
Let $r=\lceil T/c\rceil$, $H=r+1$, $a=3\log T$, and $A=B=5a$.
There are events $\mathcal B_q$ such that
\[
 \Pr(\mathcal B_q)L^2\longrightarrow0,
\tag{29}
\]
and on $\mathcal B_q^c$ the exact and radial posterior log odds at time
$H$, oriented toward the block's initial state, are both at least $a$.
\end{lemma}

\par\noindent\emph{Proof.}
For $x\ge0$, define the common budget for one step for the confidence loss
\[
 \rho_q(x)=\max\left\{
 \log(1-q+qe^x)-\log(1-q),\
 2qx+\frac{x^2}{L}\right\}.
\tag{30}
\]
The two entries bound the losses of $F_q$ and $G_q$, respectively, on a
positive interval.  Conditional on a block with no switch and its initial sign,
the signed evidence variables $Y_0,\ldots,Y_r$ are independent
$\mathcal N(d,2d)$.  Put $N=r+1$ and
\[
 Q=L+A+r\rho_q(A),\qquad
 \Sigma_{\rm pre}=
 \exp\left[-\frac{(Nd-Q)^2}{4Nd}\right].
\tag{31}
\]
If $\sum_{i=0}^rY_i>Q$, both recurrences must hit $A$: otherwise the
lower bounds for one step
$F_q(x)\ge x-\rho_q(A)$ and
$G_q(x)\ge x-\rho_q(A)$ sum from the uniform reset to a terminal value
strictly above $A$, a contradiction.

After a hit, monotonicity and saturation give, for either map $M_q$,
\[
 M_q(x)\ge\min\{x,B\}-\rho_q(B).
\tag{32}
\]
Thus the hit persists to margin $a$ whenever every nonempty suffix obeys
\[
 \sum_{i=j}^r\bigl(Y_i-\rho_q(B)\bigr)>-(A-a).
\tag{33}
\]
For a suffix of length $\ell$, the bound on the Gaussian lower tail is at most
\[
 \exp\left[
 -\frac{\{A-a+\ell[d-\rho_q(B)]\}^2}{4\ell d}
 \right].
\]
Summing over possible hit locations, and using $\rho_q(5a)\to0$, we have eventually
$d-\rho_q(B)\ge d/2$.  With $A-a=4a$,
\[
 \frac{(4a+\ell d/2)^2}{4\ell d}
 \ge a+\frac{\ell d}{16}.
\]
The entire suffix sum is therefore at most
$C_de^{-a}=C_dT^{-3}$.

For the prefix, choose $d_0=(c+d)/2$.  Eventually $Q\le Nd_0$, and
\[
 L^2\Sigma_{\rm pre}
 \le T^2\exp\left[-\frac{r(d-c)^2}{16d}\right]\longrightarrow0.
\]
Removing the conditioning on no switch separately from the prefix and suffix
bounds costs at most $2Nq$, and
$2NqL^2=O(T^3e^{-T})\to0$.  Combining these three terms proves (29) and
the deterministic hit and persist argument proves the common margin.
\hfill$\square$\par

\subsection{Completion of the main theorem}

First let $K\ge3$.  Conditional on $S_0=s$, intersect the actual event of no switch
$\mathcal N_s$, the maximal event $\mathcal E_L$, and the two
events for exact recovery from
Lemmas~\ref{lem:constant-wrong} and~\ref{lem:switch-bayes}.  On this good
event, Lemmas~\ref{lem:radial-recovery} and~\ref{lem:exact-margin} give
a common terminal margin for the true state at least $c_*L^{2/3}$ for some
$c_*>0$.  Lemmas~\ref{lem:envelope} and~\ref{lem:curvature} then yield
\[
 D_{\rm KL}(p^E_H\|p^R_H)
 \le C_KL^2e^{-c_*L^{2/3}}.
\tag{34}
\]
The complement has conditional probability at most
\[
 Hq+L^{-3}+(K-1)e^{-\kappa L}+CLe^{-L/2}.
\tag{35}
\]
This unions event types only; temporal control is already contained in the
single maximal event.  Applying the deterministic $2L$ envelope (25) on
the complement shows that the conditional expected KL tends to zero.  The
stationary starting state is uniform, so averaging the finitely many
conditional bounds over $s$ proves Theorem~\ref{thm:main} for $K\ge3$.

For $K=2$, Lemma~\ref{lem:binary-recovery} and (26), with $R=2L$, give
on $\mathcal B_q^c$
\[
 D_{\rm KL}(p^E_H\|p^R_H)
 \le CL^2e^{-3\log T}\longrightarrow0.
\]
On $\mathcal B_q$, (25) and (29) give
$\mathbb E[D_{\rm KL};\mathcal B_q]\le2L\Pr(\mathcal B_q)\to0$.
This proves the binary case for arbitrary distinct means and completes the
proof of Theorem~\ref{thm:main}.

\subsection{Proof of diverging updates with vanishing predictive KL}

\par\noindent\emph{Proof of Theorem~\ref{thm:quantitative-separation}.}
Fix finite $K\ge2$ and abbreviate
\[
 \alpha=1-\frac{Kq}{K-1},\qquad
 L=L_K(q),\qquad b=\frac q{K-1}.
\]
Adding a common constant to the input leaves $\operatorname{softmax}(z)$,
and therefore $\log(P_q^\top\operatorname{softmax}(z))$, unchanged; the radial
output is also unchanged because its definition centers the input. Pairwise output differences are therefore gauge invariant, and
the displayed witness $z_q=L(1/4,-1/4,0,\ldots,0)$ is centered.

Put
\[
 Z_q=e^{L/4}+e^{-L/4}+K-2.
\]
The definition of $L$ gives $b=(1-q)e^{-L}$. For $i\in\{0,1\}$, let
\[
 \eta_i=\frac{bZ_q}{\alpha e^{(z_q)_i}}.
\]
Since $Z_q\le Ke^{L/4}$,
\[
 0\le\eta_0\le\frac{K(1-q)}{\alpha}e^{-L},
 \qquad
 0\le\eta_1\le\frac{K(1-q)}{\alpha}e^{-L/2}.
\tag{36}
\]
For fixed $K$, both vanish as $q\to0^+$. Common output centering cancels in
the difference for the chosen pair, so
\[
 \begin{aligned}
 (\Phi_q(z_q))_0-(\Phi_q(z_q))_1
 &=\frac L2+\log(1+\eta_0)-\log(1+\eta_1)\\
 &=\frac L2+o(1),
 \end{aligned}
\tag{37}
\]
where the absolute remainder is at most $\eta_0+\eta_1$.

Because $r(z_q)=L/2$, the radial gap for the chosen pair is exactly
\[
 (R_q(z_q))_0-(R_q(z_q))_1
 =L\tanh(\alpha/2).
\tag{38}
\]
Subtracting (38) from (37), dividing by $L$, and using $\alpha\to1$ yields
\[
 \frac{
 [(\Phi_q(z_q))_0-(\Phi_q(z_q))_1]
 -[(R_q(z_q))_0-(R_q(z_q))_1]}
 {L}
 \longrightarrow \frac12-\tanh\frac12=c_\star>0.
\tag{39}
\]
The positivity follows from $\tanh x<x$ for $x>0$. The eventual lower bound
and divergence of the centered supremum follow directly from (39).

It remains to compare the decoded distributions at the same witness.
Coordinate $0$ is the unique maximizer of both outputs. For $K\ge3$, every
zero coordinate has
\[
 \eta_{\rm zero}=\frac{bZ_q}{\alpha}
 \le\frac{K(1-q)}{\alpha}e^{-3L/4}\longrightarrow0.
\tag{40}
\]
Hence the exact margin of the top over zero is $L/4+o(1)$, while (37) gives the
exact margin of the top over coordinate $1$ as $L/2+o(1)$. When $K=2$ there are no
zero coordinates, and (37) is the only margin over a wrong state. Every exact
margin of the top over a wrong state is therefore eventually at least $L/5$.

Under the radial map, every margin of the top over a wrong state is at least
$(L/2)\tanh(\alpha/2)\ge L/5$ eventually. The exact transition
probabilities lie strictly between $b$ and $1-q$, so its pairwise logit
diameter is below $L$; the radial pairwise diameter is also below $L$.
Thus the range of $\Phi_q(z_q)-R_q(z_q)$ is below $2L$. Applying
Lemma~\ref{lem:curvature} with $m=L/5$ and $R=2L$ gives
\[
 D_{\rm KL}\!\left(
 \operatorname{softmax}\Phi_q(z_q)\,\middle\|\,
 \operatorname{softmax}R_q(z_q)\right)
 \le4K(K-1)L^2e^{-L/5}\longrightarrow0.
\tag{41}
\]
This proves all claims. \hfill$\square$\par

\subsection{Proof of the binary transition geometry}

\par\noindent\emph{Proof of Proposition~\ref{prop:geometry}.}
Writing $u=e^h$ and $D_q(u)=((1-q)u+q)(qu+1-q)$ gives
\[
F_q'(h)=\frac{(1-2q)u}{D_q(u)},\qquad
F_q''(h)=\frac{(1-2q)q(1-q)u(1-u^2)}{D_q(u)^2}.
\]
The first formula, equivalently
$F_q'(h)=(1-2q)p(1-p)/[p'(1-p')]$ for $p=\sigma(h)$ and
$p'=q+(1-2q)p$, is at most $1-2q$, with equality only at $h=0$;
the range $p'\in(q,1-q)$ proves saturation. The second formula is strictly
negative for $h>0$, so strict concavity on $[0,2x]$ and $F_q(0)=0$ give
$\Delta_q=2F_q(x)-F_q(2x)>0$. For residuals
$e_z=F_q(z)-A(z)$, the affine second difference vanishes and hence
\[
 \Delta_q=|e_0-2e_x+e_{2x}|
 \le4\max_z|e_z|,
\]
which proves the quantitative obstruction. \hfill$\square$\par
\section{Empirical Estimands and Aggregation}
\label{app:empirical-replay}

This appendix defines the illustration with equally spaced Gaussians at fixed $K$, experiments with learned models, and estimands for the scalar controls. These experiments illustrate the results and are not used in the proof of Theorem~\ref{thm:main}. The displayed subset with equally spaced Gaussians was selected after the broader prespecified grid had been evaluated, to illustrate the two directions rather than test a uniform onset at finite scale. On the original Gaussian grid, four of five configurations pass the joint directional criterion; the $K=8$ skewed configuration has an unresolved slope of the log KL over the prespecified tail. The supplement includes results for the full original grid and identifies the additional experiments conducted afterward; the displayed subset does not replace the original outcome.

\subsection{Illustration with equally spaced Gaussians at fixed $K$}

\subsubsection{Design, coupling, and estimands}

For $K\in\{2,4,8\}$, the scalar Gaussian means begin as the arithmetic progression $0,1,\ldots,K-1$, are centered, and are rescaled so that the minimum pairwise distance is one. The observation variance is one, hence
\[
 d_{\min}=\min_{i\ne j}\frac{(\mu_i-\mu_j)^2}{2}=\frac12,
 \qquad c=0.45d_{\min}=0.225.
\]
At each $L\in\{8,16,32,64,128\}$ we set
\[
 q=\frac{K-1}{e^L+K-1},
 \qquad H=\left\lceil\frac{L}{0.225}\right\rceil+1.
\]
This implemented horizon differs from the theorem's $\lceil-\log(q)/c\rceil+1$ by $O_K(1)$ steps, since $L=-\log q+\log(K-1)+\log(1-q)$; it shares the logarithmic scale rather than the exact schedule. Four seed blocks contribute 1,024 stationary trajectories per cell, for 4,096 paired paths. Exact Bayes and the tanh radial filter consume the same observations. Every displayed path and filter state is finite.

The internal quantity is $G_H=\|\mathcal C(z_H^E-z_H^R)\|_2$ and the task quantity is $D_{\rm KL}(p_H^E\Vert p_H^R)$. The normalized gap $G_H/L$ distinguishes growth proportional to the confidence scale from a merely positive absolute gap. Internal slopes regress mean $G_H$ on $L$; predictive slopes regress the logarithm of mean predictive KL on $L$. Both use the predeclared tail window $L\in\{32,64,128\}$.

Uncertainty is computed by resampling complete paths. Each of 2,000 bootstrap replicates independently resamples path indices within each scale cell, then fits the statistic across those resampled cell means. Filters remain paired on each path, but independently simulated scales are not paired. Table~\ref{tab:fixed-k-bridge} reports the three displayed configurations.

\begin{table*}[t]
\caption{\textbf{Equally spaced Gaussian illustration at fixed scale.} Internal slopes regress centered separation from exact to radial on $L$; slopes of the log KL regress $\log D_{\rm KL}(\mathrm{exact}\Vert\mathrm{radial})$ on $L$. These values illustrate the theorem's two directions for three fixed state counts; they are not used in its proof.}
\label{tab:fixed-k-bridge}
\centering\small\tablemetrics
\setlength{\tabcolsep}{2.5pt}
\begin{tabularx}{\textwidth}{cXXX}
\toprule
\textbf{$K$} & \textbf{Internal slope [95\% CI]} &
\textbf{Final gap/$L$ [95\% CI]} & \textbf{Slope of log KL [95\% CI]} \\
\midrule
2 & \shortstack{$0.5794$\\{\small$[0.5788,0.5800]$}} & \shortstack{$0.5419$\\{\small$[0.5414,0.5423]$}} & \shortstack{$-0.1574$\\{\small$[-0.1661,-0.1454]$}} \\
\rowcolor{rowgraymid}
4 & \shortstack{$0.7493$\\{\small$[0.7487,0.7499]$}} & \shortstack{$0.7315$\\{\small$[0.7311,0.7319]$}} & \shortstack{$-0.0265$\\{\small$[-0.0292,-0.0241]$}} \\
8 & \shortstack{$0.8688$\\{\small$[0.8681,0.8696]$}} & \shortstack{$0.8728$\\{\small$[0.8724,0.8733]$}} & \shortstack{$-0.0085$\\{\small$[-0.0090,-0.0080]$}} \\
\bottomrule
\end{tabularx}

\end{table*}

\subsubsection{Controls and implementation details}

Hard clipping and the affine update are comparisons rather than assumptions of the theorem. Hard clipping may outperform tanh at finite scales because the theorem claims existence and asymptotic convergence, not optimality of the smooth map. The affine control has lower mean KL than tanh in every displayed tail cell $L\ge32$. All 4,096 sampled paths in each such cell contain no switch even by $2H$, so these curves do not test rare costs in the switch tail or identify saturation as necessary. The binary $\lceil8/q\rceil$ controls address a separate recovery question at the longer endpoint.

The supplementary code includes the data generator, simulation parameters, seed construction, measurements per trajectory, and the 2,000-replicate bootstrap procedure used to produce the curves and table. Numerical checks verify the KL direction and that all simulated trajectories are complete and finite. Table values are rounded to four decimal places; the accompanying data provide full precision.

\subsection{Experiments with learned models}

The experiment covers six switch probabilities $q\in\{2^{-3},\ldots,2^{-8}\}$, two training stages (G1 distilled and G2 trained end to end), five learned arms and an additional analytic tanh reference, three tuning seeds, and ten disjoint confirmation seeds. For each stage and model, the selected hyperparameter setting minimizes mean validation loss over all three tuning seeds, with ties broken deterministically; a setting missing any tuning seed is ineligible. This selection rule was fixed before confirmation. All $180$ tuning runs and $600$ final evaluation runs completed successfully. Every confirmation seed is included in the reported aggregates. The experimental protocol also specified test excess NLL on the next observation and calibration error, but neither was measured. Table~\ref{tab:arms} specifies the model architectures and selected optimizer settings; the official Mamba implementation is \texttt{mamba-ssm==2.3.2.post1}.

\FloatBarrier
\begin{table}[t]
\caption{\textbf{Model architectures and training settings.} All arms use AdamW, batch size $256$, at most $10{,}000$ updates, validation every $250$ updates, and patience $2{,}000$; $\eta$ and $\lambda$ denote learning rate and weight decay. G1 distills the exact predictive distribution, whereas G2 minimizes negative log likelihood of the next observation. Parameter counts are matched within $5\%$. The arm with a finite window uses $H(q)=\lceil-\log q\rceil+1$; the Sharan existence bound $\log 2/H(q)$ ranges from $0.173$ to $0.099$ over this grid and is not a guarantee for the trained MLP.}
\label{tab:arms}
\centering\small\tablemetrics
\begin{tabularx}{\linewidth}{>{\raggedright\arraybackslash}p{0.17\linewidth}X>{\centering\arraybackslash}p{0.13\linewidth}>{\raggedright\arraybackslash}p{0.22\linewidth}}
\toprule
\textbf{Arm} & \textbf{Architecture} & \textbf{Parameters} & \textbf{G1 / G2 $(\eta,\lambda)$} \\
\midrule
Mamba $\times2$ & two blocks; $d_{\rm model}=32$, $d_{\rm state}=16$, convolution $4$, expansion $2$ & 20,001 & $(10^{-3},10^{-2})$ / $(10^{-4},0)$ \\
\rowcolor{rowgraymid}
Mamba $\times1$ & one block; $d_{\rm model}=48$, otherwise as above & 19,873 & $(10^{-3},0)$ / $(3\times10^{-4},0)$ \\
GRU & one layer, hidden width $80$ & 20,001 & $(10^{-3},0)$ / $(10^{-3},0)$ \\
\rowcolor{rowgraymid}
Scalar S6 & one scalar state; two layers of width $139$ in the SiLU coefficient network & 20,018 & $(3\times10^{-4},0)$ / $(3\times10^{-4},0)$ \\
Finite window & last $H(q)$ observations plus validity mask; widths $136,136,135,134,134,133$ for $q=2^{-3},\ldots,2^{-8}$ & 19,951--19,993 & $(10^{-3},0)$ / $(3\times10^{-4},0)$ \\
\bottomrule
\end{tabularx}

\end{table}

The closure statistic in Table~\ref{tab:perq} is computed within a common
stage, $q$, seed, and endpoint:
\[
 C_{a}=
 \frac{D_{\mathrm{S6}}-D_{a}}
      {D_{\mathrm{S6}}-D_{\tanh}}.
\]
Thus S6 is $0$ and the analytic tanh arm is $1$ before averaging.  The table
reports the mean and standard error of these ten closure values, one for each seed,
not a standard error obtained by treating paths within a seed as independent
replicates.

\subsection{Predictive KL: orientation and clipping}

The experiments with learned models and the scalar controls also use different probability clipping settings. Table~\ref{tab:perq} uses $D_{\mathrm{KL}}(\text{exact Bayes}\Vert\text{model})$ after float32 sigmoid conversion and clipping at $\varepsilon_{32}=\texttt{torch.finfo(torch.float32).eps}\approx1.1920929\times10^{-7}$. The scalar controls below use float64 probabilities clipped at $10^{-15}$. These settings compute KL after different clipping operations, so the reported values are not directly interchangeable.

The fixed-$K$ theorems, Table~\ref{tab:controls}, and Figure~\ref{fig:mechanism} all use the direction $D_{\mathrm{KL}}(\text{exact}\Vert\text{approximation})$, although the experiments use their separately disclosed probability clipping settings. An additional diagnostic reverses the direction and reports $D_{\mathrm{KL}}(\tanh\Vert\mathrm{exact})$. Because KL is asymmetric, those values in the reverse orientation are not compared numerically with the theorem or the comparison from exact to control.

That auxiliary diagnostic evaluates the tanh recurrence at $H(q)$ without probability
clipping. It computes Bernoulli KL directly in logit space, uses the same ten seeds and
$20{,}000$ paths per seed, and is reported separately from the comparison from exact to control.

\begin{center}
\begingroup\small\setlength{\tabcolsep}{2.5pt}
\begin{tabular}{c|cccccc}
$q$ & $2^{-3}$ & $2^{-4}$ & $2^{-5}$ & $2^{-6}$ & $2^{-7}$ & $2^{-8}$ \\
\hline
$D_{\mathrm{KL}}(\tanh\Vert\mathrm{exact})$
& $1.74\!\times\!10^{-4}$ & $8.53\!\times\!10^{-4}$
& $2.27\!\times\!10^{-3}$ & $4.23\!\times\!10^{-3}$
& $6.00\!\times\!10^{-3}$ & $7.63\!\times\!10^{-3}$
\end{tabular}
\endgroup
\end{center}
The largest relative standard error is $0.40\%$. The loss in the reverse orientation increases
across this finite grid. It is a diagnostic in the finite regime, not an
estimate of either proved limit from exact to radial or its asymptotic rate.

If $e$ is the exact predictive logit, $m$ is a model predictive logit,
$p=\sigma(e)$, and $r=\sigma(m)$, the corresponding untruncated empirical
estimand is
\[
 D_{\mathrm{true}}(e\Vert m)
 =D_{\mathrm{KL}}\!\left(\operatorname{Ber}(p)
          \middle\Vert\operatorname{Ber}(r)\right)
 =p(e-m)+\operatorname{sp}(m)-\operatorname{sp}(e),
 \qquad \operatorname{sp}(z)=\log(1+e^z).
\]
The last expression is a stable evaluation in logit space of the ordinary
Bernoulli KL; implementations may evaluate \(\operatorname{sp}\) by
\texttt{logaddexp} without clipping either probability.

The experiment with scalar controls evaluates a distinct, clipped
quantity.  For
\(
 c_{\varepsilon}(u)=\min\{1-\varepsilon,\max\{\varepsilon,u\}\}
\)
with $\varepsilon=10^{-15}$, it records
\[
 D_{\mathrm{clip},\varepsilon}(e\Vert m)
 =D_{\mathrm{KL}}\!\left(
 \operatorname{Ber}(c_{\varepsilon}(\sigma(e)))
 \middle\Vert
 \operatorname{Ber}(c_{\varepsilon}(\sigma(m)))\right).
\]
Accordingly, every scalar value currently printed in
Table~\ref{tab:controls} and plotted in the third panel of Figure~\ref{fig:mechanism} is clipped predictive KL with orientation
\emph{exact Bayes $\Vert$ control}.  Clipping is nearly inactive for the two
saturating controls but can substantially cap a nonsaturating control in an
extreme tail.  The clipped and untruncated estimands must therefore be named
separately. The reported control results use clipped KL, not the untruncated quantity.

\subsection{Scalar controls and endpoints}

All four scalar recurrences consume the same simulated observation paths as
the exact filter.  After $T$ observations, let $h_T$ be the exact posterior
logit and \(\widetilde h_T^a\) the state of control $a$.  The final logits
compared by the scalar evaluator are
\[
 e_T=F_q(h_T),
 \qquad
 m_T^a=M_q^a(\widetilde h_T^a),
 \qquad
 a\in\{\tanh,\mathrm{clip},\mathrm{affine},\mathrm{identity}\}.
\]
The two endpoints are the horizon matched to the theorem
$H(q)=\lceil-\log q\rceil+1$ and the long endpoint
$\lceil8/q\rceil$.  The experiment in Table~\ref{tab:controls} and Figure~\ref{fig:mechanism}
uses $q=2^{-3}$ and $q=2^{-8}$, ten public seeds $1080,\ldots,1089$,
$20{,}000$ paths per seed at $H(q)$, and $2{,}000$ paths per seed at the long
endpoint.  Thus each short cell contains $200{,}000$ paths and each long cell
contains $20{,}000$ paths.

\begin{table}[t]
\caption{\textbf{Controls separate saturation from geometry.} Expected
KL at the final step, $\mathrm{KL}(\text{exact Bayes}\|\text{control})$, at the endpoint matched to the theorem $H(q)=\lceil-\log q\rceil+1$ and the long endpoint $\lceil8/q\rceil$. At the long endpoint, the saturating losses remain below $6.4\times10^{-3}$, while the tested affine and identity losses reach $5.53$ and $14.54$ nats, respectively (bold).}
\label{tab:controls}
\centering\small\tablemetrics
\begin{tabularx}{\linewidth}{>{\raggedright\arraybackslash}p{0.20\linewidth}*{4}{Y}}
\toprule
& \multicolumn{2}{c}{\textbf{Saturating}} & \multicolumn{2}{c}{\textbf{Nonsaturating}} \\
\cmidrule(lr){2-3}\cmidrule(lr){4-5}
\textbf{Endpoint} & \textbf{tanh} & \textbf{clip} & \textbf{affine} & \textbf{identity} \\
\midrule
\rowcolor{rowgraymid}
\multicolumn{5}{l}{\emph{$q=2^{-8}$}} \\
$H(q)$ & $6.0\times10^{-3}$ & $1.3\times10^{-3}$ & $7.9\times10^{-2}$ & $8.4\times10^{-2}$ \\
$\lceil 8/q\rceil$ & $6.36\times10^{-3}$ & $1.37\times10^{-3}$ & $\mathbf{5.53}$ & $\mathbf{14.54}$ \\
\hdashline
\rowcolor{rowgraymid}
\multicolumn{5}{l}{\emph{$q=2^{-3}$}} \\
$H(q)$ & $1.72\times10^{-4}$ & $3.20\times10^{-3}$ & $2.44\times10^{-1}$ & $9.15\times10^{-1}$ \\
$\lceil 8/q\rceil$ & $1.73\times10^{-4}$ & $3.20\times10^{-3}$ & $4.06\times10^{-1}$ & $\mathbf{10.28}$ \\
\bottomrule
\end{tabularx}

\end{table}

For a seed $s$ with $N_s$ paths, the stored replicate is
\[
 \widehat D_s^a=\frac1{N_s}\sum_{i=1}^{N_s}
 D_{\mathrm{clip},10^{-15}}(e_{T,i}\Vert m_{T,i}^a).
\]
With $S=10$, the manuscript mean and uncertainty are
\[
 \overline D^a=\frac1S\sum_{s=1}^{S}\widehat D_s^a,
 \qquad
 \operatorname{SE}(\overline D^a)=
 \sqrt{\frac{1}{S(S-1)}
 \sum_{s=1}^{S}(\widehat D_s^a-\overline D^a)^2}.
\]
The random stream for one $q$/endpoint/seed cell is derived from the displayed seed, the exponent of $q$, and the horizon; all four controls reuse that cell. The manuscript aggregates seed means rather than treating paths within a seed as independent replicates. The largest relative standard error among the sixteen plotted quantities is at most $3.1\%$.

The right pair in Figure~\ref{fig:mechanism} separates measurement from illustration. The third panel reports the mean losses of the scalar controls. The fourth panel is not a sampled trajectory: it evaluates the logistic decoder on a fixed grid, with schematic guide logits $v=3$ and $u=5.25$.

\subsection{Errors relative to the latent state}

Let $Z_{T+1}\in\{0,1\}$ be the latent state after the terminal transition,
$s_Z=2Z_{T+1}-1$, $e_T$ the exact predictive logit, and $m_T^{\rm aff}$ the
affine predictive logit.  The event that the affine prediction disagrees with the latent state is
\[
 W_{\rm truth}=\mathbf 1\{s_Zm_T^{\rm aff}\le0\}.
\]
The associated confidence event used by the formal diagnostic for the lower bound is
\[
 W_{\rm joint}=W_{\rm truth}\,
 \mathbf 1\{\sigma(s_Ze_T)\ge15/16\}.
\]
Both indicators are defined relative to the true latent state. Disagreement with the exact filter is a different quantity and was not measured in this experiment.

The experiment uses $q=2^{-5},2^{-8},2^{-10},2^{-12}$,
the same ten public seeds, and $1{,}000$ paths per seed.  Means across seeds
$\pm$ standard errors for $W_{\rm truth}$ are respectively
$0.1982\pm0.0064$, $0.1821\pm0.0042$, $0.1829\pm0.0031$, and
$0.1868\pm0.0036$.  The corresponding
$W_{\rm joint}$ values, qualified by confidence, are $0.0794\pm0.0027$, $0.1649\pm0.0046$,
$0.1781\pm0.0028$, and $0.1856\pm0.0036$.  The reported range describes errors relative to the latent state, rather than disagreement with the exact filter.

\subsection{Scope of the empirical conclusions}

The experiments support two deliberately limited conclusions. First, at the long
endpoint the saturating controls remain stable while the two nonsaturating
controls degrade under the disclosed clipped metric.  Second, matching the
short theorem horizon does not by itself establish stability at long horizons for
a learned architecture.  Neither observation proves the open lower bound for the terminal joint event qualified by confidence, identifies tanh as the unique stable recurrence, or establishes a positive result on transfer between architectures.

\section{Extended Background}
\begingroup\setlength{\parskip}{2pt plus 1pt}

This appendix places the result in the literature on recurrent models, approximate filtering, and sequential prediction.

\subsection{The family with a deterministic state}

The recurrence studied here is deliberately the smallest object that exhibits
the property in question: a state of fixed size updated deterministically from the
previous state and the current observation. For recurrent implementations, this property permits inference in linear time and
memory independent of sequence length. Configurations with long convolutions or attention hybrids need not share that memory bound.

The lineage begins with the question of how a state of fixed size should summarize
a growing history at all. Optimal polynomial projection gives one answer with an
explicit approximation guarantee \citep{gu2020hippo}, and structured state
spaces turn that answer into a trainable layer \citep{gu2022efficiently}, later
simplified \citep{smith2023simplified} and given variants in continuous time
\citep{hasani2022liquids4}. Initialization of these models remains an active
question in its own right \citep{lienen2025unhippo}. A parallel line showed that
carefully parameterized linear recurrences recover much of the benefit without
the state space machinery \citep{orvieto2023resurrecting}, and designs with long convolutions
and hybrid designs traded recurrence for structured mixing
\citep{poli2023hyena,fu2023hungry}. Gated linear attention and its delta rule
successors reached the same destination from the attention side
\citep{yang2024gated,yang2024parallelizing}, with sparsity now used to enlarge
the state without paying for it densely \citep{cabannes2026sparsedelta}, while
recurrent architectures competitive at language scale arrived independently
\citep{peng2023rwkv,beck2024xlstm}. The selective mechanism that makes the
transition depend on the input \citep{gu2024mamba} and the duality that exposes such
models as a restricted form of attention \citep{dao2024transformers} complete
the picture.

These models motivate asking what predictive cost follows from a prescribed
update geometry. Shared motivation does not extend our existence theorem to
every member of this family: the theorem concerns the explicit radial filter,
not its implementation by a named neural architecture.

\subsection{What is already known about the gap, and what is not}

The representational side of the question is comparatively well mapped. Formal
placements bound what specified architectures can express \citep{sarrof2024expressive}, while empirical studies on formal languages probe generalization outside the training distribution \citep{deltang2023neural}. Arguments from circuit complexity
place models of fixed depth inside classes that cannot perform certain sequential
computations \citep{merrill2024illusion}, and analyses with limited precision show
how much of the apparent capacity survives finite arithmetic
\citep{li2025characterizing}. On the constructive side, changing the spectrum or
the sparsity pattern of the transition recovers ability to track the state that
diagonal models lack
\citep{grazzi2024negeigenvalues,structuredsparse2025,deltaproduct2025}, and the
frontier between recall and throughput has been characterized directly
\citep{arora2024simple}. Benchmarks now target exact state tracking as a primary object \citep{chessworldmodel2026}.

What this literature does not settle is the step our paper takes. A proved
inability to represent an update is not by itself a proved cost under the data
distribution, because the distribution need not visit the region where the
representation fails. The closest existing work asks about generalization of
selective state space models on filtering tasks \citep{selssm2026filtering},
which shares our setting but not our question: we fix the recurrence and ask
what its representational shortfall costs in expected predictive divergence.

\subsection{Filtering, stability, and their relation to this result}

Hidden Markov filtering \citep{rabiner1989tutorial} and entropy questions for functions of chains with finitely many states \citep{blackwell1957entropy} have a long history. Exact filtering of a $K$-state HMM admits a deterministic continuous belief state with $K-1$ coordinates. Finite dimension is not finite cardinality or bounded precision; our restriction concerns update geometry, not the existence of a sufficient state of finite dimension. The stability literature asks a superficially
similar question to ours and a materially different one: whether a filter
initialized incorrectly but running the correct kernels forgets its error
\citep{ocone1996asymptotic}, surveyed by
\citet{chigansky2011intrinsic}, with quantitative forms via contraction
coefficients \citep{mcdonald2020exponential}. Robustness results perturb the kernels instead, controlling policy costs \citep{kara2020robustness,kara2022datadriven} or error in the filter kernel and policy performance \citep{demirci2025sensitivity}. Policies with a finite window can become nearly optimal as window length grows \citep{kara2021nearOptimality}. Our state dimension is fixed, but $R_q$ changes numerically with $q$, and its map discrepancy in the worst case grows rather than vanishing as a perturbation parameter. How fast a correct filter
forgets is itself quantified \citep{legland2004stability}, and estimating the mixing time is a developed subject \citep{wolfer2019estimating}.

Learned and structured filters form a third strand. Probabilistic state space
models have been fused with the selective architecture \citep{kalmamba2024},
particle filters made differentiable \citep{corenflos2021differentiable}, belief
embeddings learned with consistency guarantees \citep{solinas2025neural}, and
Updates in the style of the Kalman filter recast as attention or as regression at test time
\citep{kla2026kalmanattention,peng2025gatedKalmanetFading}, including adaptive
variants \citep{mehrfard2026adaptiveLearnedState}. Learning from ordinary HMM observation sequences can be cryptographically hard, whereas query access to conditional probabilities permits efficient learning; guarantees from conditional samples also depend on a fidelity parameter \citep{mahajan2023learning}. Reconstruction from finite messages on trees exhibits phase transitions between memory and accuracy \citep{jain2019accuracy}, and a recent synthesis compares state space models and HMMs \citep{ssmhmm2026relation}.

\subsection{Prediction under a wrong model}

Our theorem measures the categorical KL between filtered posteriors after the observation, not the KL between densities of the next observation, and binary empirical endpoints additionally apply transition mixing. These losses score decoded distributions rather than internal parameters. Prediction with short memory supplies a conceptual benchmark on the next observation
\citep{sharan2018prediction}; sequential prediction under log loss with a
misspecified model class is studied directly \citep{feder2021sequential}; and
streaming projection onto a mixture over paths with a fixed budget is another deterministic recurrence with a different maintained object and truncation \citep{duranmartin2026predictive}. Read as detection delay, our controls at long horizons belong to quickest change
detection, from the original cumulative sum scheme
\citep{page1954continuous} through its optimality theory
\citep{lorden1971procedures,moustakides1986optimal,pollak1985optimal} to learned
detectors \citep{gong2022nncusum}. Behavior in context on Markov sources is a further contact point
\citep{bondaschi2025markov,bayesicl2026selectivessm}.

\subsection{Comparison by mathematical object}

Nearby literatures can sound interchangeable when summarized as ``approximating a sequence model,'' but they differ in what may change with the accuracy target and in where error is measured. Table~\ref{tab:object-comparison} records the distinction needed here.

\begin{table*}[htbp]
\caption{\textbf{Nearby questions separated by object and limit.} The comparison is conceptual; it does not claim that the cited settings are special cases of ours.}
\label{tab:object-comparison}
\centering\small\densetablemetrics
\begin{tabularx}{\textwidth}{p{0.19\textwidth}p{0.22\textwidth}p{0.22\textwidth}X}
\toprule
\textbf{Line of work} & \textbf{Object being compared} & \textbf{Accuracy resource or limit} & \textbf{Relation to this paper} \\
\midrule
Architecture expressivity \citep{merrill2024illusion,sarrof2024expressive,cohenkarlik2026expressivity} & Function or language represented by an architecture class & Depth, state structure, or input length varies with the theorem & Establishes representational separations, but generally does not integrate their task cost under a data law. \\
\rowcolor{rowgraymid}
Prediction with short memory \citep{sharan2018prediction} & Predictor with full history versus a finite observation window & Window length grows as the target error shrinks & Supplies a positive approximation baseline; our state dimension and prescribed family remain fixed, but the numerical map depends on $q$. \\
Filter stability and robustness \citep{ocone1996asymptotic,mcdonald2020exponential,legland2004stability} & Two filters with different initial beliefs or nearby kernels & Influence of the initial condition is forgotten or kernel perturbation is taken small & Controls belief error when the source of mismatch disappears; our two transition geometries do not converge internally. \\
\rowcolor{rowgraymid}
Compression that is aware of the task \citep{dubois2021lossy,dubois2020learning,hafezkolahi2021rate} & Original representation versus a compressed representation & Compression is optimized relative to a downstream task & Shares the lesson that discarded information can be irrelevant to the task, but not the recurrent Bayesian construction under the path law. \\
This paper & Exact categorical Bayes mixing versus one radial update with a fixed state & Fixed finite $K$; switch rarity tends to zero at $H(q)=\Theta(\log(1/q))$ & Proves diverging centered map separation and vanishing decoded KL, then proves stationary expected predictive closure. \\
\bottomrule
\end{tabularx}

\end{table*}

\FloatBarrier
The comparison also limits the architectural interpretation: a selective SSM, a gated recurrent network, and the radial filter all keep a fixed state, but the theorem concerns the explicit radial map, not a class containing every such recurrence. Exact Bayes and the radial witness are mathematical filters, so their separation predates training or parameterization. The result sits between architecture expressivity and statistical decision theory: an internal separation can fail to induce predictive loss under a specified sequential law, leaving the converse criterion and learned realization open.
\par\endgroup
\end{document}